\documentclass[10pt,twocolumn,letterpaper]{article}
\usepackage{float}
\usepackage{iccv}
\usepackage{times}
\usepackage{epsfig}
\usepackage{graphicx}
\usepackage{amsmath}
\usepackage{amsthm}
\usepackage{amssymb}
\usepackage{comment}
\usepackage{rotating}
\usepackage{paralist}
\DeclareMathOperator{\Tr}{tr}
\newcommand{\LL}{\mathcal{L}}
\newtheorem{theorem}{Theorem}

\usepackage[pagebackref=true,breaklinks=true,letterpaper=true,colorlinks,bookmarks=false]{hyperref}

\iccvfinalcopy 

\def\iccvPaperID{2027} 

\ificcvfinal\fi
\begin{document}

\title{Characterizing and Improving Stability in Neural Style Transfer}

\author{
	Agrim Gupta, Justin Johnson, Alexandre Alahi, and Li Fei-Fei \\[1mm]
    Department of Computer Science, Stanford University \\
    {\tt\small agrim@stanford.edu} \hspace{2pc} {\tt\small \{jcjohns,alahi,feifeili\}@cs.stanford.edu}
}

\maketitle

\begin{abstract}
Recent progress in style transfer on images has focused on improving the quality of stylized images and speed of  methods. However, real-time methods are highly unstable resulting in visible flickering when applied to videos. In this work we characterize the instability of these methods by examining the solution set of the style transfer objective. We show that the trace of the Gram matrix representing style is inversely related to the stability of the method. Then, we present a recurrent convolutional network for real-time video style transfer which incorporates a temporal consistency loss and overcomes the instability of prior methods. Our networks can be applied at any resolution, do not require optical flow at test time, and produce high quality, temporally consistent stylized videos in real-time.
\end{abstract}


\section{Introduction}
\def\mywidth{0.17\textwidth}
\def\ratio{0.15}
\begin{figure}[t]
\centering
\center
    \includegraphics[width=\linewidth]{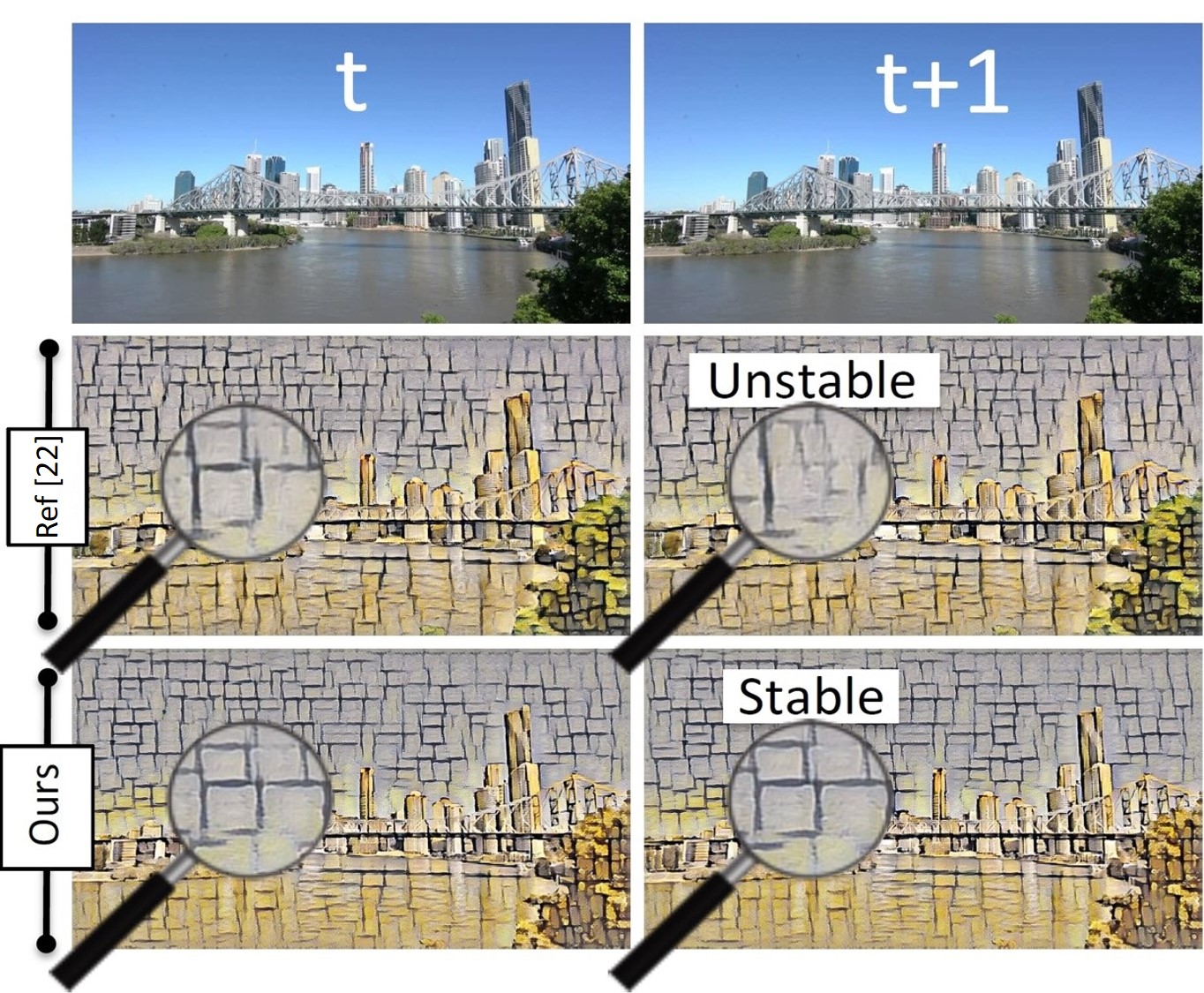}
  \caption{State-of-the-art real-time style transfer networks (e.g., \cite{john}) are highly unstable. Consecutive video frames that are visually indistinguishable to humans (top) can result in perceptibly different stylized images (middle). In this work, we characterize such instability and propose a recurrent convolutional network that produces temporally consistent stylized video in real-time (bottom).}
  \vspace{-1mm}
  \label{fig:pull}
\end{figure}
Artistic style transfer of images aims to synthesize novel images combining the \emph{content} of one image with the \emph{style} of another. This longstanding problem~\cite{ashikhmin2001synthesizing,efros2001image,hertzmann1998painterly,hertzmann2000painterly} has recently been revisited with deep neural networks~\cite{gatys}. Subsequent work has improved speed~\cite{john,ulyanov2016texture}, quality~\cite{UlyanovVL17,gatys2016controlling,wang2016multimodal}, and modeled multiple styles with a single model~\cite{supercharging}. 

Recent methods for style transfer on images fall into two categories.
Optimization-based approaches~\cite{gatys,gatys2016controlling} solve an optimization
problem for each synthesized image; they give high-quality results but can take
minutes to synthesize each image. Feedforward methods~\cite{john,ulyanov2016texture,supercharging} train neural networks to
approximate solutions to these optimization problems; after training they
can be applied in real-time. However, all these methods are highly unstable resulting in visible flickering when applied to videos; see Figure~\ref{fig:pull}. Ruder \etal~\cite{ruder2016artistic} extends the optimization-based approach from images to videos. Their method produces high-quality stylized videos, but is too slow for real-time application. 

In this paper, our goal is to perform feedforward style transfer of videos while matching the high-quality results produced by the optimization-based methods. Recent methods for style transfer use \emph{Gram matrix}
of features to represent image style: stylized images are synthesized by matching the Gram
matrix of the style image. We find that the trace of the style image's Gram matrix is closely
related to the pixel instability shown in Figure~\ref{fig:pull}. Specifically, the solution
set of the Gram matrix matching objective is a sphere with radius determined by the
trace of the style image's Gram matrix. Due to the nonconvexity of this objective, minor
changes in the content image can pull the synthesized image to different solutions of this
Gram matrix matching objective function. If all the solutions are close together
(small trace), then different solutions will still result in similar
stylized images (no instability). However, if the solutions are far apart (large trace), then
different solutions will result in very different stylized images (high instability).

Based on this insight, we propose a method which greatly improves the stability of
feedforward style transfer methods, synthesizing high-quality stylized video.
Specifically, we use a recurrent convolutional network for video stylization,
trained with a \emph{temporal consistency loss} which encourages the network to find
similar solutions to the Gram matrix matching objective at each time step.
Our contributions in this paper are twofold: 

\begin{compactitem}
\item[(i)] First, we characterize the instability of recent
style transfer methods by examining the solution set of the style transfer objective,
showing an inverse relation between the trace of a style's Gram matrix and its stability.
Our characterization applies to all neural style transfer
methods based on Gram matrix matching. 

\item[(ii)] Second, we propose a recurrent convolutional network for real-time video style transfer
which overcomes the instability of prior methods. Inspired by~\cite{ruder2016artistic},
we incorporate a loss based on optical flow encouraging the network to produce
temporally consistent results. Our method combines the speed of feedforward image
stylization~\cite{john,ulyanov2016texture,supercharging} with the quality and temporal 
stability of optimization-based video style transfer~\cite{ruder2016artistic}, giving a $1000\times$
speed improvement for stable video style transfer without sacrificing quality.
\end{compactitem}

\section{Related Work}
\textbf{Texture Synthesis.} Texture synthesis is closely related to style transfer; the goal is to infer a generating process from an input texture to enable further production of samples of the same texture. Earlier attempts in computer vision to address the problem of texture synthesis can be divided into two distinct approaches: parametric and non-parametric. Parametric approaches compute global statistics in feature space and sample images from the texture ensemble directly \cite{portilla2000parametric,zhu2000exploring,heeger1995pyramid,de1997multiresolution}. Non-parametric approaches estimate the local conditional probability density function and synthesize pixels incrementally. Methods based on this approach generate new textures either by re-sampling pixels \cite{efros1999texture,wei2000fast} or whole patches \cite{kwatra2003graphcut,efros2001image} of the original texture. \par 
Parametric methods are based on Julesz \cite{julesz1962visual} characterization of textures where two images are said to have same texture if they have similar statistical measurements over a feature space. Gatys \etal~\cite{gatys2015texture} build on the seminal work by Portilla and Simoncelli \cite{portilla2000parametric} by using feature space provided by high performing neural networks and using Gram matrix as the summary statistic. In \cite{UlyanovVL17} the authors tackle the problem of perceptual quality in feed forward based texture synthesis by proposing instance normalization and  a new learning formulation that encourages generators to sample unbiasedly from the Julesz texture ensemble. Chen and Schmidt \cite{ChenS16f} build on the work done in texture transfer by proposing a novel``style swap" based approach where they create the activations of the output image by swapping patches of the content image with the closest matching style activation patches. The swapped activation is then passed through a inverse network to generate styled image. Their optimization formulation is more stable than \cite{john,ulyanov2016texture} making it particularly suitable for video applications. Though their approach has generalization power \cite{gatys} and is stable, it can't be used for real time video application due to run-time being of the order of seconds. \par 
\textbf{Style Transfer.} Gatys et al.~\cite{gatys} showed that high quality images can be generated by using feature representations from high-performing convolutional neural networks. Their optimization based approach produces perceptually pleasing results but is computationally expensive. Johnson \etal~\cite{john} and Ulyanov \etal~\cite{ulyanov2016texture} proposed feed-forward networks which were a thousand times faster than \cite{gatys} and could produce stylized images in real-time. However, each style requires training of a separate feed-forward network and the perceptual quality of the images is inferior compared to the optimization based approach. Dumoulin \etal~\cite{supercharging} proposed a conditional instance normalization layer to address this issue, allowing one network to learn multiple styles. This results in a simple and efficient model which can learn arbitrarily different styles with considerably fewer parameters without compromising on speed or perceptual quality as compared to \cite{john,ulyanov2016texture}.\par
\textbf{Optical Flow.} Accurate estimation of optical flow is a well studied problem in computer vision with variety of real-world applications. Classical approach for optical flow estimation is by variational methods based on the work by Horn and Schunck \cite{horn1981determining}. Convolutional neural networks (CNNs) have been shown to perform as good as state of the art optical flow detection algorithms. FlowNet \cite{flow2,fischer2015flownet} matches the performance of variational methods and introduces novel CNN architectures for optical flow estimation. A full review of optical flow estimation is beyond the scope of this paper and interested readers can refer to \cite{aggarwal1988computation,sun2014quantitative,barron1994performance} \par
\textbf{Style Transfer on Videos.} Artistic stylization of images is traditionally studied under the label of non-photorealistic rendering. Litwinowicz \cite{litwinowicz1997processing} was one of the first to combine the idea of transferring brush strokes from impressionist painting to images and using optical flow to track pixel motion across video frames to produce temporally coherent output video sequences. Hays and Essa \cite{hays2004image} build on this to add further optical and spatial constraints to overcome flickering and scintillation of brush strokes in \cite{litwinowicz1997processing}. Hertzmann \cite{hertzmann1998painterly} improves on the perceptual quality of images by presenting  techniques for painting an image with multiple brush sizes, and for painting with long, curved brush strokes and later extend this work to videos in \cite{hertzmann2000painterly}. \par
Recently Ruder \etal~\cite{ruder2016artistic} extended the optimization based approach in \cite{gatys} by introducing optical flow based constraints which enforce temporal consistency in adjacent frames. They also proposed a multi-pass algorithm to ensure long term consistency in videos. Their algorithm produces extremely good results in terms of temporal consistency and per frame perceptual quality but takes a few minutes to process one frame.
\section{Stability in Style Transfer}
\label{sec:stability}

\subsection{Style Transfer on Images}
We use the style transfer formulation from \cite{gatys}, which
we briefly review.
Style transfer is an image synthesis task where we receive as input a \emph{content image}
$c$ and a \emph{style image} $s$. The output image $p$ minimizes the objective
\begin{equation} \label{eq:1}
\mathcal{L}(s,c,p)  = \lambda_c\ \mathcal{L}_c(p, c) + \lambda_s\ \mathcal{L}_s(p, s),
\end{equation}
where $\mathcal{L}_c$ and $\mathcal{L}_s$ are the \emph{content reconstruction loss} and \emph{style reconstruction loss} respectively; $\lambda_c$ and $\lambda_s$ are scalar hyperparameters governing their importance. \par 
Content and style reconstruction losses are defined in terms of a convolutional
neural network $\phi$; we use the VGG-19~\cite{Simonyan14c} network pretrained on ImageNet.
Let $\phi_j(x)$ be the \textit{j}$^{\text{th}}$ layer network activations of shape $C_j \times H_j \times W_j$ for 	image $x$. Given a set of content layers $\mathcal{C}$ and style layers $\mathcal{S}$, the content and style reconstruction losses are defined as:
\begin{equation} \label{eq:2}
\mathcal{L}_c(p, c) = \sum_{j \in \mathcal{\mathcal{C}}} \frac{1}{C_j H_j W_j} \Vert \phi_j(p) - \phi_j(c) \Vert^2_2,
\end{equation}
\begin{equation}\label{eq:3}
\mathcal{L}_s(p, s) = \sum_{j \in \mathcal{\mathcal{S}}} \frac{1}{C_j H_j W_j} \Vert G(\phi_j(p)) - G(\phi_j(s)) \Vert^2_F,
\end{equation}
where $G(\phi_j(x))$ is a $C_j\times C_j$ \emph{Gram matrix} for layer $j$ activations given by
$G(\phi_j(x)) = \Phi_{jx}\Phi_{jx}^T$,
where $\Phi_{jx}$ is a $C_j\times H_jW_j$ matrix whose columns are the $C_j$-dimensional features of $\phi_j(x)$.

Rather than forcing individual pixels of output image to match content and style images,
the content and style reconstruction losses encourage the generated image to match the high-level
features of the content image and the feature correlations of the style image.

\subsection{Gram Matrix and Style Stability}

As shown in Figure~\ref{fig:pull}, imperceptible changes in the content image $c$ can result
in drastically different stylized images $p$. However, we observe that this instability is not
uniform across all styles. Some style, such as \emph{Composition XIV} (see Figure~\ref{fig:pull})
are highly unstable, while others such as \emph{The Great Wave} (see Figure~\ref{fig:short-term}) 
show less qualitative instability.

To understand how instability depends on the style image, we consider only style loss for a single layer. Then the style transfer network minimizes the objective (dropping the subscript $j$ for notational convenience):

\begin{equation}
\begin{aligned}\label{eq:5}
& \underset{ G(\phi(p))}{\text{min}}
& & \frac{1}{C H W} \Vert G(\phi(p)) - G(\phi(s)) \Vert^2_F, \\
& \underset{\Phi_p}{\text{min}}
& & \Vert\Phi_p\Phi_p^T - \Phi_s\Phi_s^T \Vert^2_F.
\end{aligned}
\end{equation}
As motivation, consider the simple case $C=H=W=1$; then Equation~\ref{eq:5} reduces to $(\Phi_p^2-\Phi_s)^2$,
which is a nonconvex function with minima at $\Phi_p=\pm \Phi_s$, shown in Figure~\ref{fig:proof} (left).
Similarly, when $C=H=1,W=2$, shown in Figure~\ref{fig:proof} (right), the minima lie on
a circle of radius $\Phi_s$. In both the cases, the minima lie at a distance $\Phi_s$ away from the origin.
This observation holds in general:
\begin{theorem}
Let $\gamma$ be a sphere centered at the origin with radius $\Tr (\Phi_s\Phi_s^T)^\frac{1}{2}$. Then, $\Phi_p$ minimizes the objective $J(\Phi_p) = \Vert\Phi_p\Phi_p^T - \Phi_s\Phi_s^T \Vert^2_F$ iff \ $\Phi_p \in \gamma$. 
\end{theorem}
\begin{proof}
Suppose that $\Phi_p=\Phi_s$; then $\Phi_p\Phi_p^T=\Phi_s\Phi_s^T$ and $J(\Phi_p) = 0$ so $\Phi_p$
minimizes the objective.
Now let $\Phi_p$ be any minimizer of $J$; then $J(\Phi_p)=0$, so $\Phi_p\Phi_p^T=\Phi_s\Phi_s^T$ and thus
$\Tr\Phi_p\Phi_p^T=\Tr\Phi_s\Phi_s^T$ and so $\Phi_p\in\gamma$. \\

\begin{figure}
\begin{center}
   \includegraphics[width=0.46\linewidth]{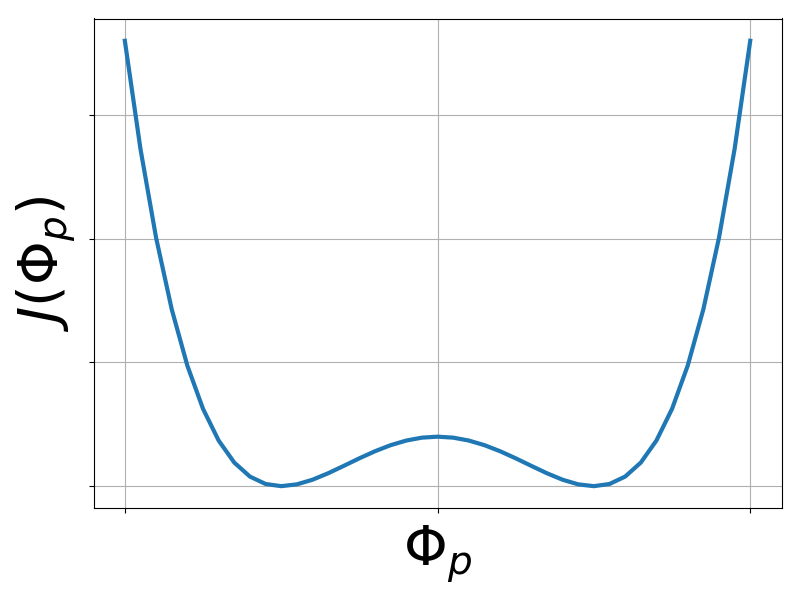} 
   \includegraphics[width=0.46\linewidth]{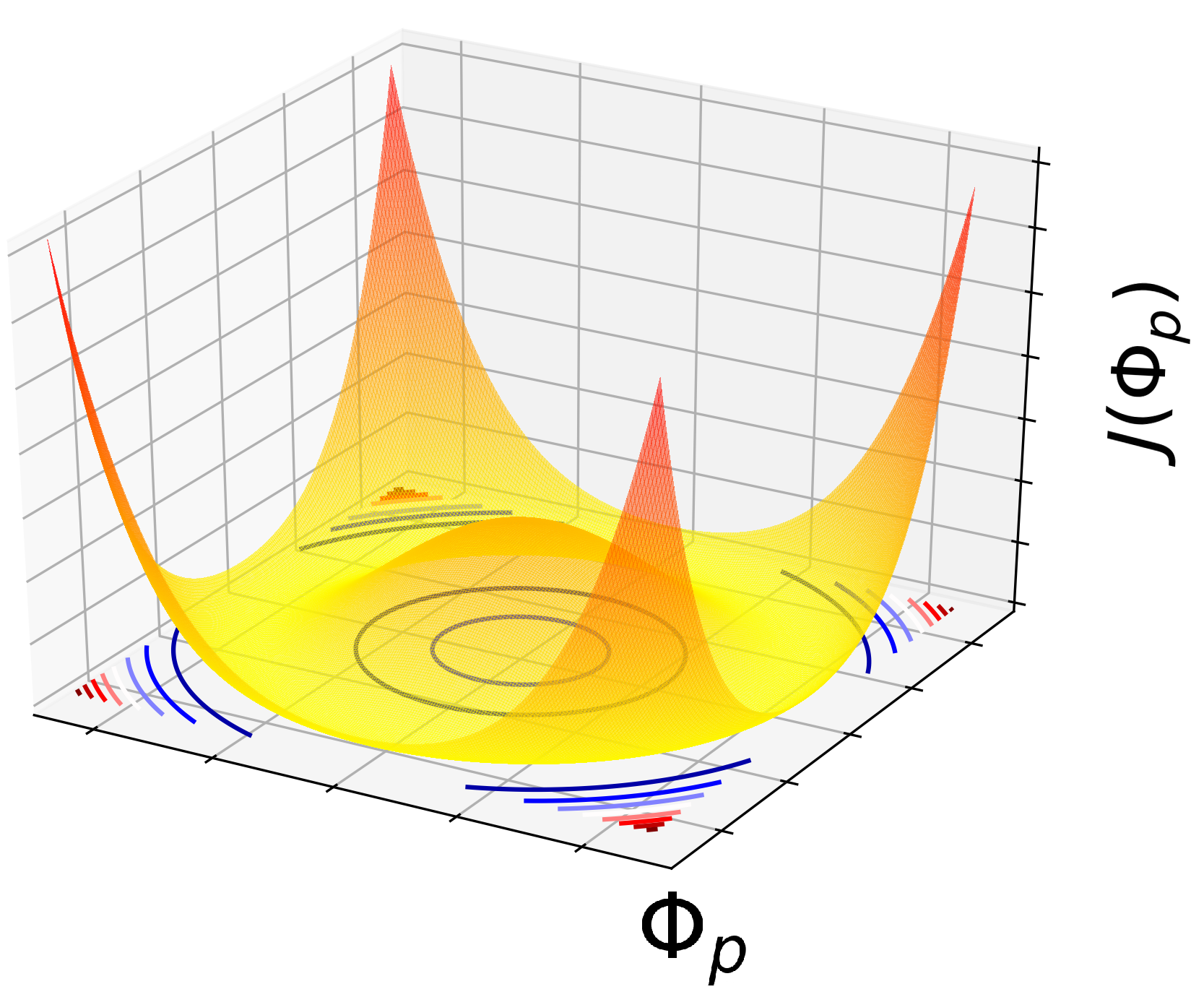} 
\end{center}
   \caption{Visualization of the minimization objective in Equation \ref{eq:5} for 1-D and 2-D cases. For 1-D the two solutions differ by $2S$. Similarly, for 2-D case the solutions lie on a circle of radius $S$. This generalizes to the $N$-D case where the solutions lie on a $N$-D sphere centered at the origin with radius $\Tr(G(\phi(p)))^\frac{1}{2}$.}
\label{fig:proof}
\end{figure}

Now let $\Phi_p\in\gamma$. Since $J(\Phi_s)=0$ we know $\Phi_s\in\gamma$; thus there exists an orthogonal rotation matrix $U$ such that $\Phi_p=\Phi_sU$. Then we have $\Phi_p\Phi_p^T=\Phi_sUU^T\Phi_s^T=\Phi_s\Phi_s^T$, so $J(\Phi_p)=0$ and thus $\Phi_p$
minimizes $J$.
\end{proof}

This result suggests that styles for which the Gram matrix trace $\Tr\Phi_s\Phi^T$ is large should exhibit more severe
instability, since solutions to the style reconstruction loss will lie further apart in feature space as $\Tr\Phi_s\Phi_s^T$
increases.

We empirically verify the relationship between $\Tr\Phi_s\Phi_s^T$ and the stability of style transfer
by collecting a small dataset of videos with a static camera and no motion; the only 
differences between frames are due to imperceptible lighting changes or sensor noise.

We then trained feedforward style transfer models on the COCO dataset~\cite{lin2014microsoft} for
twelve styles using the method of \cite{john}, and used each of these models to stylize each frame
from our stable video dataset. Due to the static nature of
the input videos, any difference in the stylized frames are due to the inherent instability of the style
transfer models; we estimate the \emph{instability} of each style as the average mean squared error
between adjacent stylized frames. In Figure~\ref{fig:relu} we plot instability vs the trace of the Gram
matrix for each style at the \texttt{relu1\_1} and $\texttt{relu2\_1}$ layers of the VGG-16 loss network;
these results show a clear correlation between the trace and instability of the styles.

\begin{figure}
	\centering
	\includegraphics[height=0.17\textwidth]{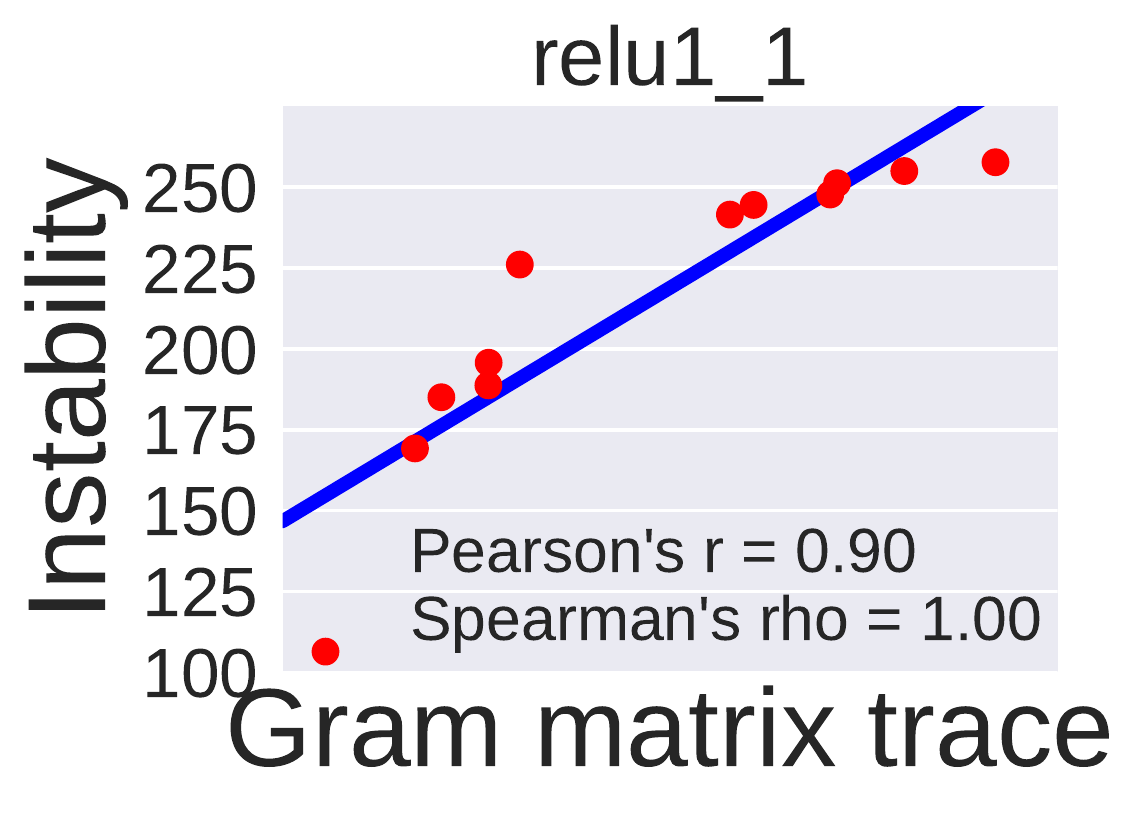}
    \includegraphics[height=0.17\textwidth]{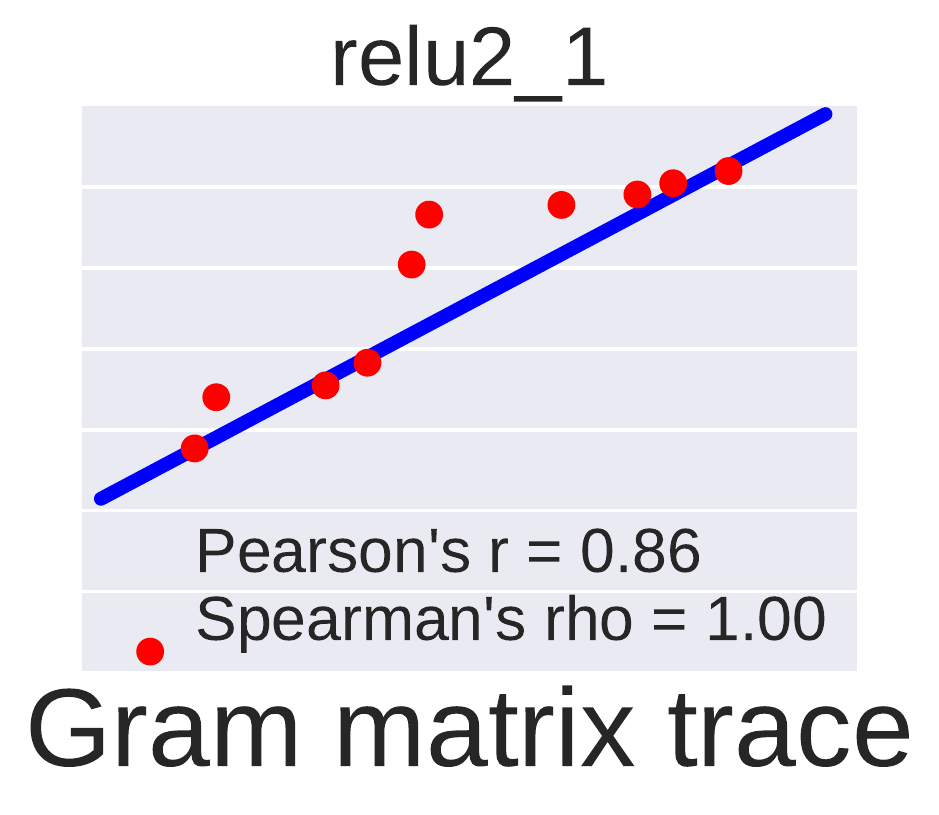}
	\caption{
		We train feedforward style transfer models for twelve styles, and define the \emph{instability}
        of a style as the mean squared error between stylized adjacent frames over a
        dataset of videos with a static camera and no motion. We also compute the trace
        of the Gram matrix at two layers of the VGG-16 loss network for each style; styles
        with larger trace tend to be more unstable.
    }
\label{fig:relu}
\end{figure}

\section{Method: Towards Stable Style Transfer}

As shown above, feedforward networks for real-time style transfer can produce unstable stylized
videos when the trace of the style's Gram matrix is large. We now present a feedforward
style transfer method that overcomes this problem, matching the speed of \cite{john} and the stability of
\cite{ruder2016artistic}.

\subsection{Overall Architecture}
Our method takes as input a sequence of content images $c_1,\ldots,c_T$ and a single style image $s$, and produces as output a sequence of stylized
images $p_1,\ldots,p_T$. Each output image $p_t$ should share content with $c_t$, share style with $s$, and be similar in
appearance to $p_{t-1}$. At each time step, the output image $p_t$ is synthesized by applying a learned \emph{style transfer network} $f_W$ to the previous stylized image $p_{t-1}$ and the 
current content image $c_t$: $p_t = f_W(p_{t-1}, c_t)$.

Similar to~\cite{john,ulyanov2016texture} we train one network $f_W$ per style image $s$. The network is trained to minimize the sum
of three loss terms at each time step:
\begin{align}
	\label{eq:method-loss}
	&\LL(W, c_{1:T}, s) \\ &= \sum_{t=1}^T(\lambda_c\LL_c(p_t, c_t) + \lambda_s\LL_s(p_t, s) + \lambda_t\LL_t(p_{t-1}, p_t)), \nonumber
\end{align}
where $\LL_c$ and $\LL_s$ are the content and style reconstruction losses from Section~\ref{sec:stability}; $\LL_t$ is \emph{temporal consistency loss} which prevents the
network output from drastically varying between time steps. The scalars $\lambda_c,\lambda_s$,
and $\lambda_t$ are hyperparameters weighting the importance of these terms.
The network $f_W$ is trained to minimize the combined loss in Equation~\ref{eq:method-loss}
on a training dataset of video sequences $\{c_{1:T}\}$ via stochastic gradient descent.

\begin{figure}
\begin{center}
\includegraphics[width=1\linewidth]{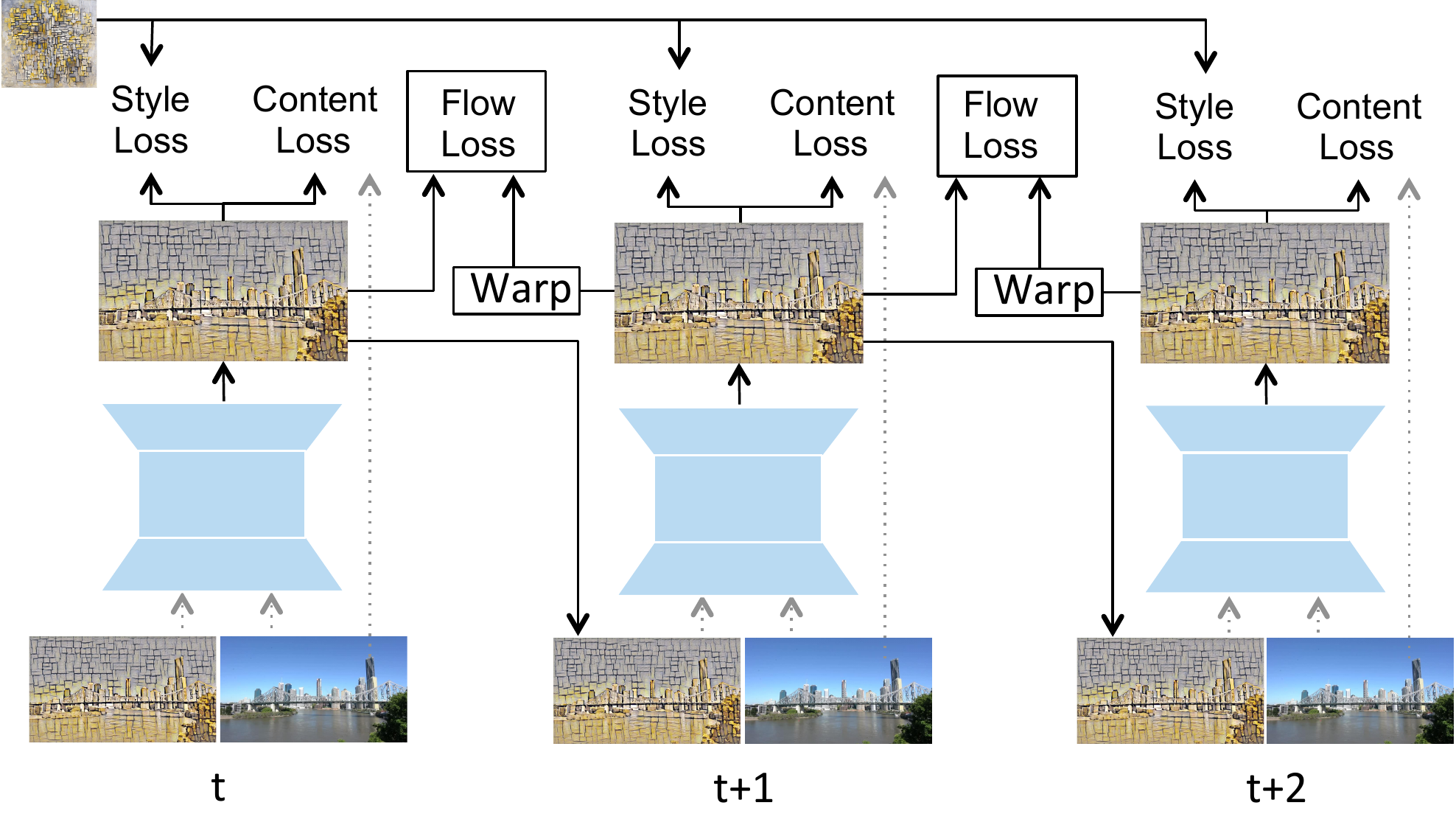}
\end{center}
   \caption{
     System overview. Our style transfer network takes as input the previous stylized image $p_{t-1}$ and the current video frame $c_t$, and produces a stylized version of the frame. The output at each timestep is fed as input at the next time step, so our system is a \emph{recurrent convolutional network}. At each time step we enforce style and content losses to ensure similarity with the input frame and style image; between each consecutive frame we enforce a \emph{temporal consistency loss} which encourages temporally stable results.
   }
\label{fig:system}
\end{figure}

\subsection{Style Transfer Network}
If our network is to produce temporally consistent outputs, then it cannot process
frames independently; it must have the capacity to examine its own previous outputs to ensure
consistency. Our networks therefore take as input both the current content image
$c_t$ and the stylized result from the previous frame $p_t=f_W(p_{t-1}, c_t)$. As shown
in Figure~\ref{fig:system}, the output from the network at each time step is fed as input to
the network at the next time step. The network $f_W$ is therefore a
\emph{recurrent convolutional network}, and must be trained via backpropagation through time~\cite{rumelhart1985learning,werbos1990backpropagation}.

The two inputs to $f_W$ are concatenated along the channel dimension, after which the architecture
of $f_W$ follows \cite{supercharging}: it is a deep convolutional 
network with two layers of spatial downsampling followed by several residual 
blocks~\cite{he2016deep} and two layers of nearest-neighbor upsampling and convolution. All 
convolutional layers are followed by instance normalization~\cite{UlyanovVL17} and ReLU 
nonlinearities.

\subsection{Temporal Consistency Loss}
By design our style transfer network can examine its own previous outputs, but 
this architectural change alone is not enough to ensure temporally consistent results.
Therefore, similar to Ruder \etal~\cite{ruder2016artistic} we augment the style and content losses with a \emph{temporal consistency loss} $\LL_t$ encouraging temporally stable results by penalizing the network when its outputs at adjacent time steps significantly vary.

The simplest temporal consistency loss would penalize per-pixel differences between
output images: $\LL_t(p_{t-1}, p_t)=\|p_{t-1}-p_t\|^2$. However, for producing high-quality stylized video
sequences we do not want stylized video frames to be exactly the same between time steps; instead we want 
brush strokes, lines, and colors in each stylized frame to transfer to subsequent frames in a manner 
consistent with the motion in the input video.

To achieve this our temporal consistency loss utilizes \emph{optical flow} to ensure that changes in output
frames are consistent with changes in input frames. Concretely, let $\mathbf{w}=(u,v)$ be the (forward) optical
flow field between input frames $c_{t-1}$ and $c_t$. Perfect optical flow gives a pixelwise correspondence
between $c_t$ and $c_{t-1}$; we want the corresponding pixels of $p_t$ and $p_{t-1}$ to match. Therefore the
temporal consistency loss penalizes the difference
\begin{equation}
	p_{t-1}(x, y) - p_{t}(x + u(x, y), y + v(x, y))
\end{equation}
for all pixel coordinates $(x,y)$. This difference can be efficiently implemented by \emph{warping}
the output frame $p_t$ using the optical flow to give a warped frame $\tilde{p_t}$, then computing the
per-pixel differences between $\tilde{p_t}$ and $p_{t-1}$. 
The use of bilinear interpolation makes this warping differentiable~\cite{jaderberg2015spatial}.

\begin{figure}
\centering
\includegraphics[width=0.4\textwidth]{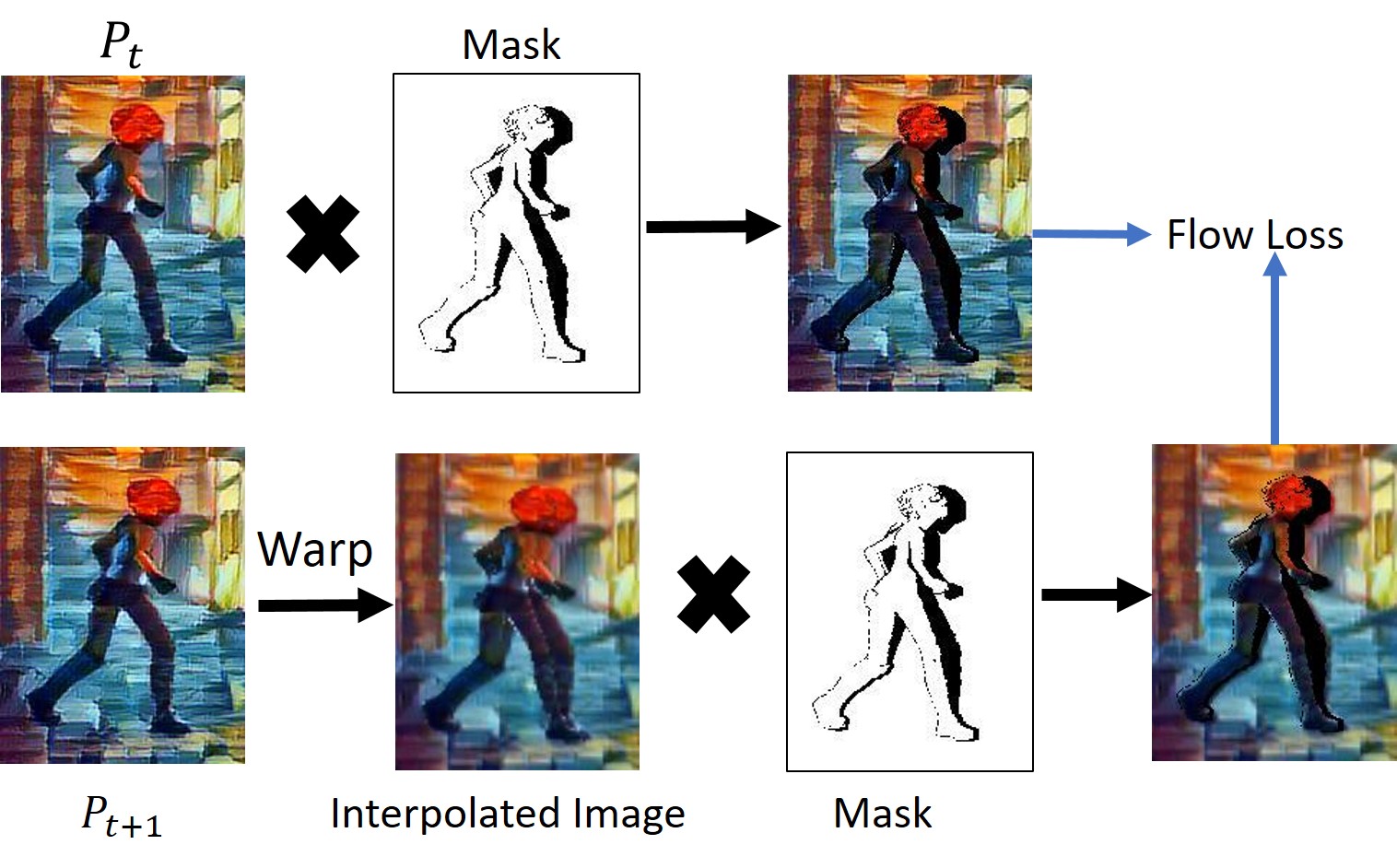}
\caption{The temporal consistency loss $\LL_t(p_{t-1}, p_t)$ warps $p_t$ using optical flow,
giving a warped frame $\tilde{p_t}$. The previous frame $p_{t-1}$ and warped frame $\tilde{p_t}$ are
multiplied by an \emph{occlusion mask}; $\LL_t(p_{t-1},p_t)$ is then the Euclidean distance between
$p_{t-1}$ and $\tilde{p_t}$.}
\label{fig:method}
\end{figure}

Due to foreground object motion, some pixels in $c_{t-1}$ may become occluded in $c_t$; likewise some pixels
which are occluded in $c_{t-1}$ may become disoccluded in $c_t$. As a result, enforcing the temporal
consistency loss between \emph{all} pixels of $\tilde{p_t}$ and $p_{t-1}$ would result in artifacts at motion
boundaries. We therefore use a ground-truth occlusion mask $m$ to avoid enforcing the temporal consistency
loss for occluded and disoccluded, giving our final temporal consistency loss:
\begin{equation} \label{eq:10}
\LL_t(p_{t-1}, p_t) =  \frac{1}{HW}\|m_t\odot p_{t-1} - m_t\odot \tilde{p_t}\|^2_F,
\end{equation}
where $m(h, w) \in [0, 1]$ is 0 in regions of occlusion and motion boundaries, $\odot$ is elementwise
multiplication, and $H,W$ are the height and width of the input frame. This loss function is summarized
in Figure~\ref{fig:method}.

Computing this loss requires optical flow and occlusion masks; however since this loss is only applied
during training, our method does not require computing optical flow or occlusion masks at test time.
 
\subsection{Implementation Details}
Following \cite{john,supercharging} we first train an image style transfer network on the COCO
dataset~\cite{lin2014microsoft} using only the content and style losses $\LL_c$ and $\LL_s$. We then
finetune the model using all three losses on the Sintel Dataset~\cite{Butler:ECCV:2012,Wulff:ECCVws:2012};
Sintel consists of rendered images and thus provides pixel-perfect optical flow and occlusion masks.
Finetuning rather than training from scratch allows for a more controlled comparison with previous methods
for feedforward image stylization. During training we resize all video frames to $512\times218$
and use random horizontal and vertical flips for data augmentation; we train for 10 epochs with BPTT for 4 
time steps using Adam~\cite{kingma2014adam} with learning rate $1\times10^{-3}$.

\section{Experiments}\label{sec:exp}

\begin{table}
  \centering
  \scalebox{0.85}{
  \begin{tabular}{|c|c|c|c|c|c|}
    \hline
    Style  
    & \multicolumn{1}{|p{1.5cm}|}{\centering Real-Time\\ Baseline \cite{john}}
    & \multicolumn{1}{|p{1.5cm}|}{\centering Optim \\ Baseline \cite{ruder2016artistic}}
    & Ours \\
    \hline
    The Wave           & 24.3 / 0.47 & \textbf{25.5} / 0.48 & 24.8 / \textbf{0.54} \\
    Metzinger          & 23.6 / 0.31 & \textbf{24.4} / 0.37 & 24.2 / \textbf{0.42} \\
    Composition XIV    & 23.8 / 0.31 & 24.0 / 0.38 & \textbf{24.2} / \textbf{0.42} \\
    Mosaic             & 23.7 / 0.31 & \textbf{24.4} / 0.37 & 24.0 / \textbf{0.39} \\
    Rain Princess      & 23.8 / 0.41 & \textbf{25.2} / 0.45 & 24.4 / \textbf{0.49} \\
    \hline
  \end{tabular}}
  \vspace{1mm}
  \caption{
  	We evaluate the stability of each method by finding corresponding $100\times100$ background patches 
    between adjacent video frames from the DAVIS dataset, then computing PSNR / SSIM between these patches
    in the stylized versions of the frames. We report mean PSNR / SSIM between stylized corresponding
    patches across all frames from all videos of the DAVIS dataset. Our method is generally more stable
    than the Real-Time baseline, and has comparable stability to the Optim baseline.
  }
  \vspace{-2mm}
  \label{table:psnr-ssim}
\end{table}

Our experiments show that our method results in stylized video sequences with comparable image
quality and stability as optimization-based methods~\cite{ruder2016artistic}, while retaining the speed advantage of feedforward methods ~\cite{john,ulyanov2016texture}.

\subsection{Baselines}
We compare our method with two state-of-the-art approaches for image and video stylization.

\textbf{Real-Time Baseline}~\cite{john,supercharging}. We train feedforward networks for image
stylization and apply the resulting network to each video frame independently. 
This method allows for high-quality image stylization in real-time, but leads to severe temporal instability 
when applied to videos.

\textbf{Optim Baseline}~\cite{ruder2016artistic}. This method explicitly minimizes an objective function
to synthesize stylized frames, leading to very stable, high-quality results, but requiring
minutes of processing per video frame. We use the single-pass algorithm from~\cite{ruder2016artistic}; their multipass algorithm improves long-term temporal consistency at the cost increased runtime.
We use the open-source code released by the authors of~\cite{ruder2016artistic}.

\begin{figure}[h]
	\centering
     \includegraphics[width=0.82\linewidth]{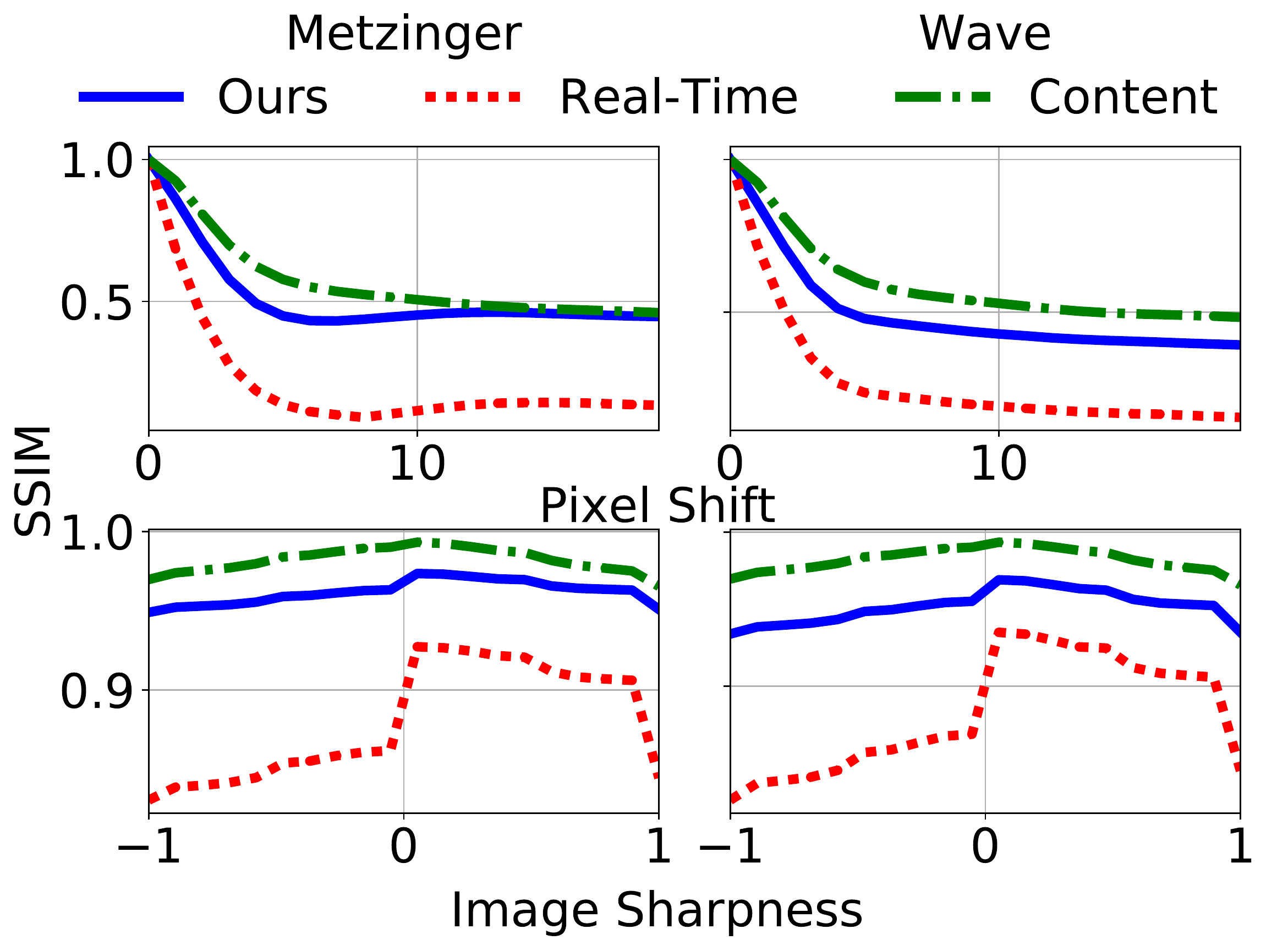}
     \caption{
     	We measure the stability of our method with respect to controlled image distortions by
    	taking an image patch, distorting it, then computing SSIM between the original and distorted patch 
        (Content); we then stylize both the original and distorted patches using our method and the
        Real-Time baseline, and compute SSIM between the stylized original and stylized distorted patches.
        Varying the magnitude of the distortion measures the stability of the method with respect to
        that distortion.
        We consider shifting the patch by 0 to 19 pixels in the input frame (top) and 
        applying blur and sharpening kernels to the patch with varying strength (bottom); repeating the 
        experiment for two styles (Metzinger, left and Wave, right). All results are averaged over 
        random $256\times256$ background patches from 100 random frames from DAVIS videos.
        Our method is much more robust to these controlled distortions.
     }
     \vspace{-1mm}
  \label{fig:stability_curve}
\end{figure}

\subsection{Datasets}
\textbf{Sintel.} The Sintel Dataset~\cite{Butler:ECCV:2012,Wulff:ECCVws:2012} consists of $35$ video
sequences from an animated movie, divided into a training set of $23$ videos and a test set of $12$ videos;
each sequence has between $20$ and $50$ frames.
We use Sintel for training since it provides pixel-perfect optical flow and occlusion masks; we show
qualitative results on the Sintel test set.

\textbf{DAVIS.} The DAVIS dataset~\cite{Perazzi2016} comprises $50$ real-world video sequences with an average
of $69$ frames per sequence. These videos include effects such as occlusion, motion-blur,
appearance change, and camera motion, making it a challenging benchmark for video stylization.
Each video frame is annotated with a ground-truth foreground/background segmentation mask.
We use this dataset for qualitatively and quantitatively evaluating all methods.

\subsection{Quantitative Evaluation}
\label{sec:stability-eval}

\textbf{Controlled Distortions.}
If a style transfer method is stable, then small changes to its input should result in modest changes
to its output. Before moving to unconstrained videos, we can first measure the stability of our method
in a more controlled setting using artificial image distortions.

Given an image patch $c$, we apply a distortion $d$ to give a distorted patch $d(c)$. We measure SSIM 
between $c$ and $d(c)$, then apply a trained style transfer network $f_W$ to both $c$ and $d(c)$, and 
measure SSIM between $f_W(c)$ and $f_W(d(c))$; we then repeat the process, varying the magnitude of the 
distortion $d$. If the network $f_W$ is robust to the distortion, then as the magnitude of $d$ increases the 
image similarity between $f_W(c)$ and $f_W(d(c))$ should decay at the same rate as the similarity between
$c$ and $d(c)$.

In Figure~\ref{fig:stability_curve} we show the results of this experiment for two types of distortion:
translation and blurring / sharpening, comparing our method against the Real-Time baseline on two styles.
In all cases, as distortion magnitude increases our
method shows a decay in image similarity comparable to that between $c$ and $d(c)$; in contrast the image
similarity for the baseline decays sharply. This shows that compared to the baseline, our method is 
significantly more robust to controlled distortions.

\begin{table}[t]
\begin{center}
\scalebox{0.8}{
\begin{tabular}{|c|c|c|c|c|c|}
\hline
Image-Size
& \multicolumn{1}{|p{1.5cm}|}{\centering Real-Time\\ Baseline \cite{john}}
& \multicolumn{1}{|p{1.5cm}|}{\centering Optim \\ Baseline \cite{ruder2016artistic}}
& Ours 
& \multicolumn{1}{|p{1.5cm}|}{\centering Speedup \\ vs \\ \cite{ruder2016artistic}}\\
\hline
$256 \times 256$ & 0.024   & 22.14 & 0.024 &  922x \\
$512 \times 512$ & 0.044   & 59.64 & 0.044 & 1355x \\
$1024 \times 1024$ & 0.141 & 199.6 & 0.141 & 1415x \\
\hline
\end{tabular}
}
\end{center}
\caption{Speed (in seconds) for our method vs the two baseline methods for processing a single video frame for varying resolutions. For the Optim baseline we use the default of 1000 iterations. Our method generates stylized video matching the speed of real-time baseline method and the temporal consistency of the Optim baseline. All methods are benchmarked on a Titan X Pascal GPU.} \label{table:speed} 
\end{table}

\textbf{Video Stability.} We aim to show that our method can stylize real-world videos, matching the stability of the optim baseline. The instability of the real-time baseline typically manifests most strongly in background image regions with
relatively little motion. To quantitatively measure this phenomenon, we use the foreground/background masks
from the DAVIS dataset.

For each video sequence in the DAVIS dataset, we find corresponding $100\times100$ pixel patches in adjacent
frames that do not intersect the foreground object as follows: We first choose a random background patch
at time $t$, then find the patch at time $t+1$ which is within 20 pixels of the first patch and maximizes the PSNR
between the two patches. We then compute PSNR and SSIM~\cite{wang2004image} between the stylized versions of these 
corresponding patches, and report the average PSNR and SSIM across all methods, videos\footnote{Running the Optim baseline for all videos and styles would take approximately 25 days on a GPU, which is computationally infeasible. Therefore, for this method we select three random videos for each style.}, and across 5 styles; results are shown in
Table~\ref{table:psnr-ssim}.

\def\mywidth{0.17\textwidth}
\def\ratio{.15}
\begin{figure}[ht!]
\centering
\scalebox{0.85}{\vbox{
\center
\hspace{-0.02\textwidth}
    \raisebox{8.5mm}{
    \begin{minipage}[c][3mm][c]{3mm}
      \centering
      \rotatebox{0}{\centering \textbf{$t_1$}}
    \end{minipage}
  }
 	\includegraphics[trim={150px 0px 450px 120px},width=\ratio\textwidth,clip=true]{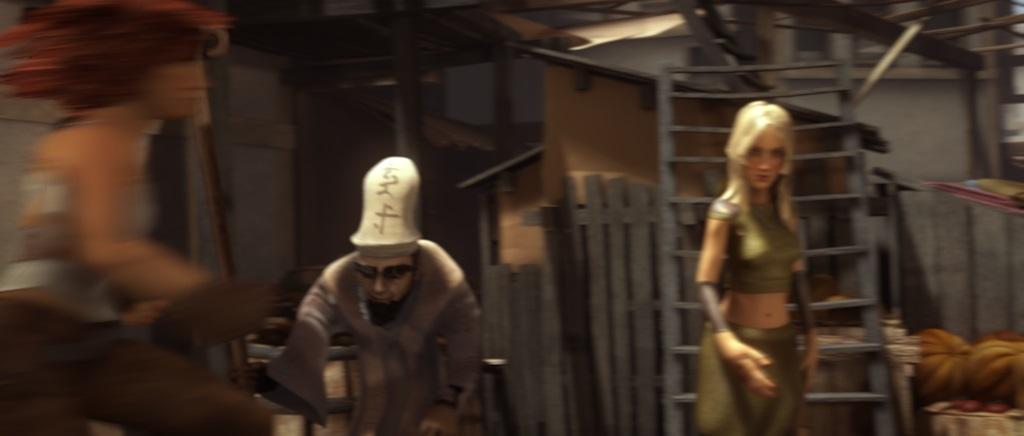}  
	\includegraphics[trim={150px 0px 450px 120px},width=\ratio\textwidth,clip=true]{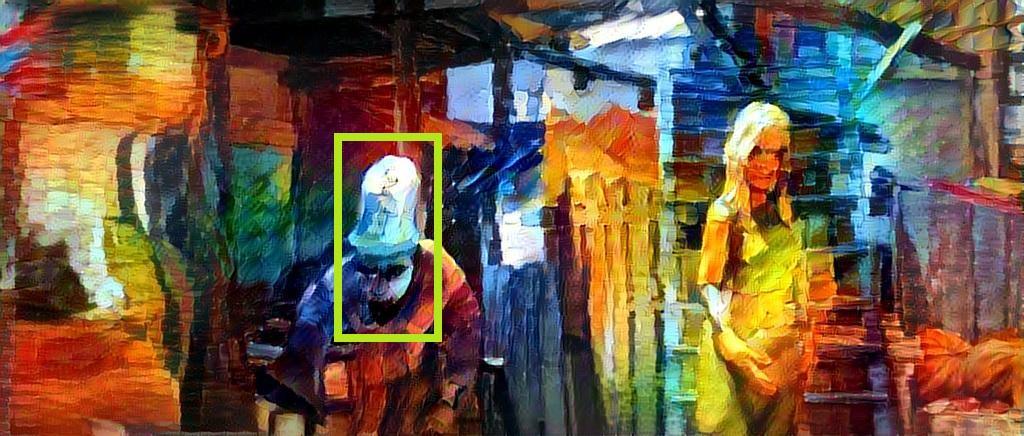}
  	\includegraphics[trim={150px 0px 450px 120px},width=\ratio\textwidth,clip=true]{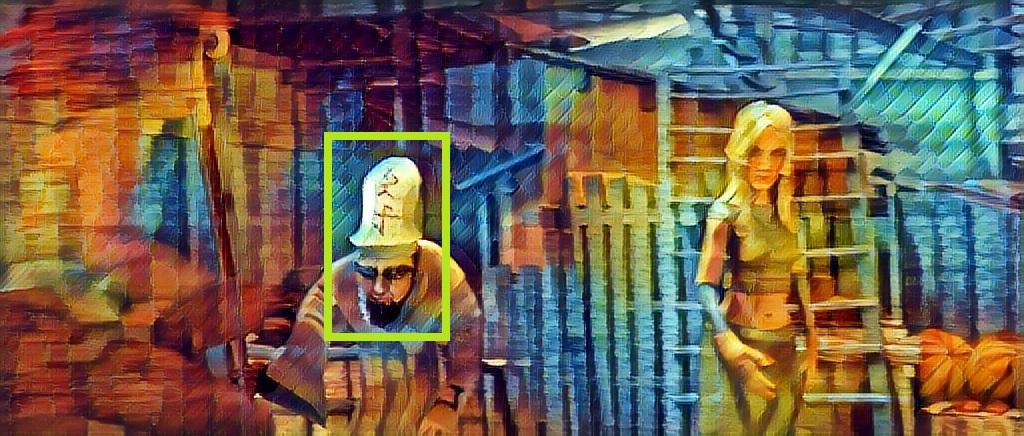} \\
\hspace{-0.02\textwidth}
    \raisebox{8.5mm}{
    \begin{minipage}[c][3mm][c]{3mm}
      \centering
      \rotatebox{0}{\centering \textbf{$t_2$}}
    \end{minipage}
  }
 	\includegraphics[trim={150px 0px 450px 120px},width=\ratio\textwidth,clip=true]{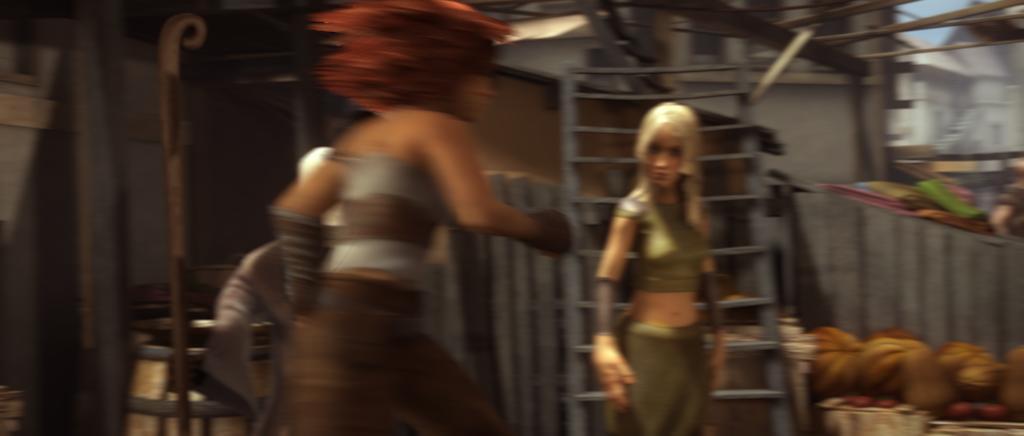}  
	\includegraphics[trim={150px 0px 450px 120px},width=\ratio\textwidth,clip=true]{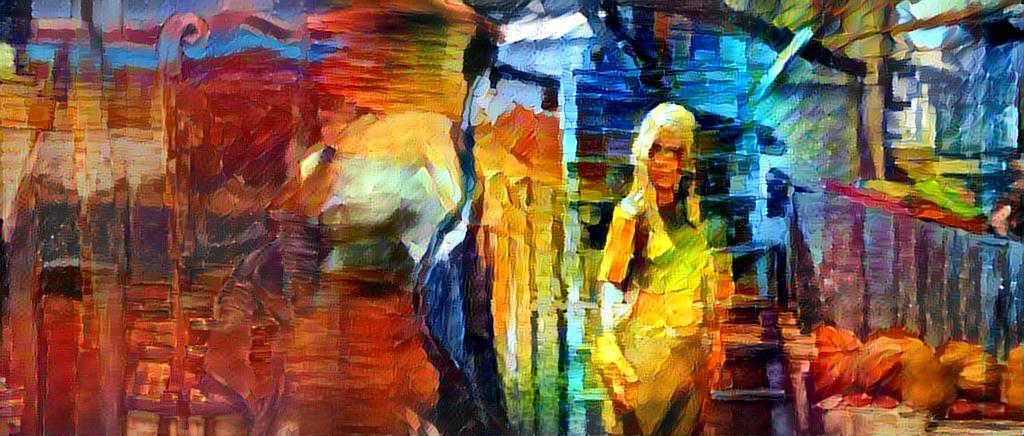}
  	\includegraphics[trim={150px 0px 450px 120px},width=\ratio\textwidth,clip=true]{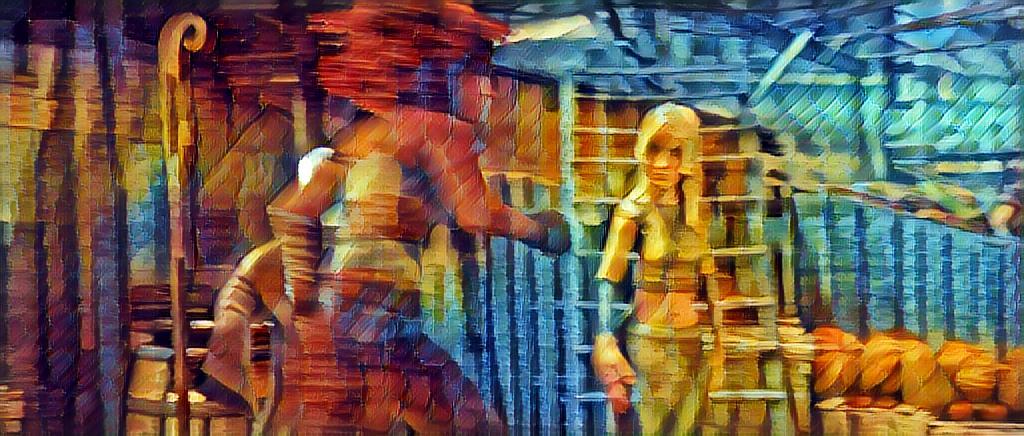}\\
\hspace{-0.02\textwidth}
    \raisebox{8.5mm}{
    \begin{minipage}[c][3mm][c]{3mm}
      \centering
      \rotatebox{0}{\centering \textbf{$t_3$}}
    \end{minipage}
  }
 	\includegraphics[trim={150px 20px 450px 100px},width=\ratio\textwidth,clip=true]{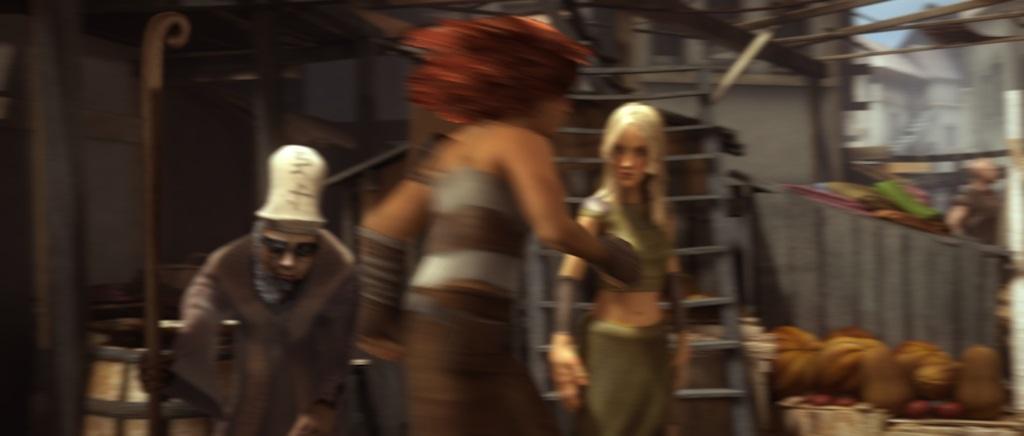}  
	\includegraphics[trim={150px 20px 450px 100px},width=\ratio\textwidth,clip=true]{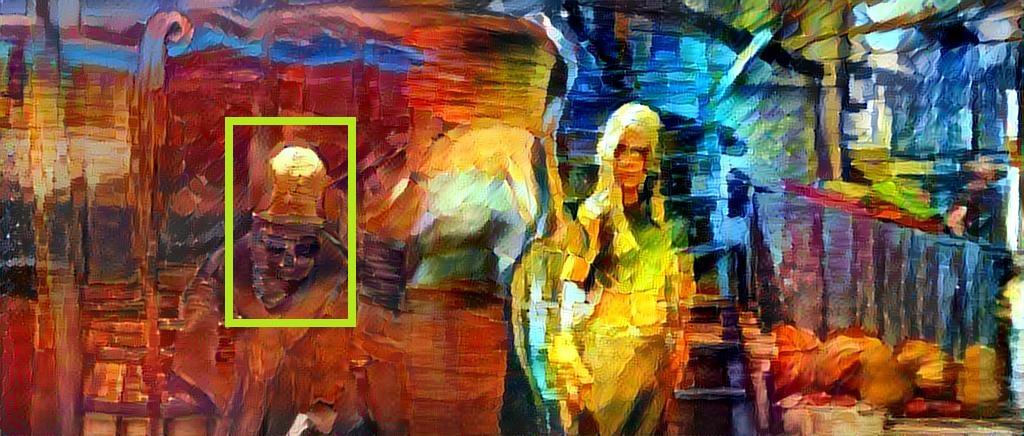}
  	\includegraphics[trim={150px 20px 450px 100px},width=\ratio\textwidth,clip=true]{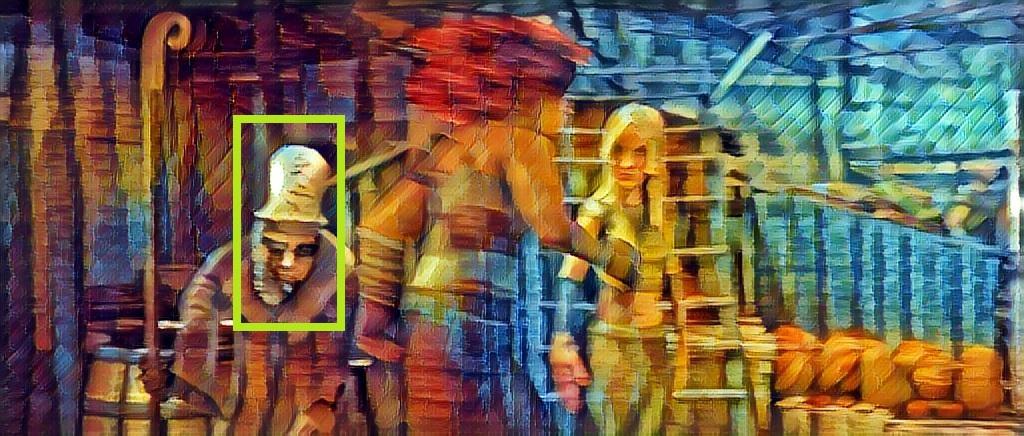}\\
    \raisebox{1mm}{
    \begin{minipage}[c][0.1mm][c]{3mm}
    \end{minipage}
  }
  \begin{minipage}{\ratio\textwidth}
    \centering \textbf{Content}
  \end{minipage}
    \begin{minipage}{\ratio\textwidth}
	\centering \textbf{Optim Baseline}
  \end{minipage}
  \begin{minipage}{\ratio\textwidth}
    \centering \textbf{Ours}
  \end{minipage} \\
  }}
  \caption{Video sequence where our method performs better than the Optim baseline. We consider three consecutive video frames. The man gets occluded in $t_2$ and on dis-occlusion in $t_3$ the style of the man is different from $t_1$ for Optim baseline. In case of our method the styles remains the same.}
  \label{fig:long-term}
\end{figure}

This experiment quantifies the instability of the real-time baseline, and shows that our method produces stylized videos with stability comparable to the Optim baseline.

\textbf{Speed.}  Table \ref{table:speed} compares the runtime of our method with the baselines for several image sizes; for the Optim baseline we exclude the time for computing optical flow.

At test-time our method produces temporally consistent stylized video frames without explicitly computing
optical flow. It matches the speed of the real-time baseline; both are roughly three orders of magnitude
faster than the Optim baseline. Our method can process images of size $512 \times 512$ at $20$ FPS, making it 
feasible to run in real-time. 

\subsection{Qualitative Evaluation}
\textbf{Short-Term Consistency.} Figure \ref{fig:short-term} shows patches from adjacent 
frames of stylized videos from the DAVIS dataset for different styles. The real-time baseline method is unable to produce the same stylization for background regions across consecutive video frames; these sudden changes between adjacent frames manifest as a severe ``flickering'' effect in video. In contrast, the Optim baseline and our method both result in consistent stylization of patches across adjacent frames, eliminating this flickering.

Figure \ref{fig:short-term} also showcases the dependence between Gram matrix trace and style instability. Both \emph{Portrait de Metzinger} and \emph{Rain Princess} style images have very high Gram matrix traces, causing the instability in the real-time baseline. The trace of the Gram matrix for \emph{The Great Wave} is much smaller, and correspondingly the real-time baseline shows less instability on this style.

\textbf{Long-Term Stability.} 
One challenging problem in video stylization is \emph{long-term stability}: When an object is occluded and then reappears, it should have the same style. Although we do not explicitly enforce long-term consistency in our loss, we find that our method nevertheless sometimes shows better long-term consistency than the Optim baseline; see Figure~\ref{fig:long-term}. In this example the man is visible at $t_1$, is occluded by the girl at $t_2$, and reappears at $t_3$. In our method the man has similar appearance at $t_1$ and $t_3$, but the baseline ``smears'' the red color of the girl onto the man in $t_3$. The authors of \cite{ruder2016artistic} also propose a multi-pass version of their algorithm which explicitly addresses this issue but it requires at least $10$ passes over the video.  

\def\mywidth{0.17\textwidth}
\def\ratio{0.13}
\begin{figure}[t]
\centering
\scalebox{0.9}{\vbox{
\center
    \raisebox{11mm}{
    \begin{minipage}[c][21mm][c]{3mm}
      \centering
      \rotatebox{90}{\centering \textbf{Content}}
    \end{minipage}
  }
  \includegraphics[trim={200px 240px 650px 10px},width=\ratio\textwidth,clip=true]{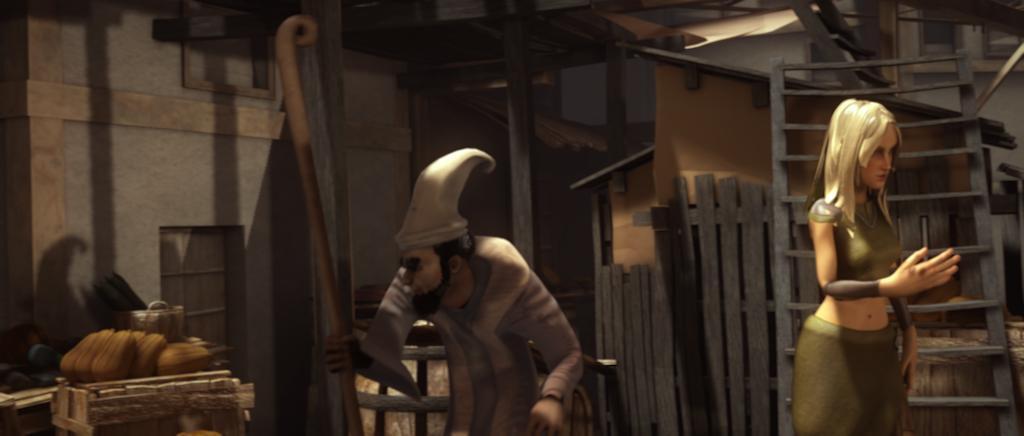}
  \includegraphics[trim={200px 240px 650px 10px},width=\ratio\textwidth,clip=true]{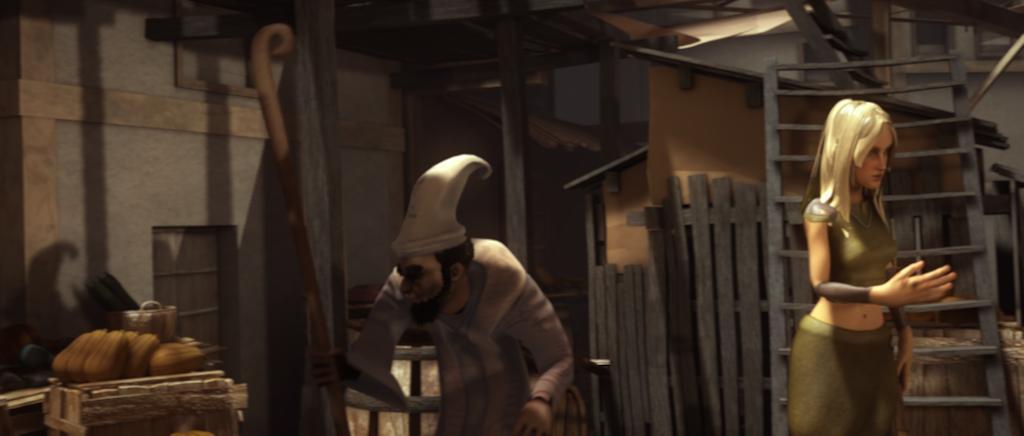}
  \raisebox{12mm}{
    \begin{minipage}{21mm}
      \centering
      \textbf{Style} \\
      \textit{Composition XIV},\\
       Piet Mondrian
    \end{minipage}
  }\\
      \raisebox{11mm}{
    \begin{minipage}[c][21mm][c]{3mm}
      \centering
      \rotatebox{90}{\centering \textbf{Optim Baseline}}
    \end{minipage}
  }
    \includegraphics[trim={200px 240px 650px 10px},width=\ratio\textwidth,clip=true]{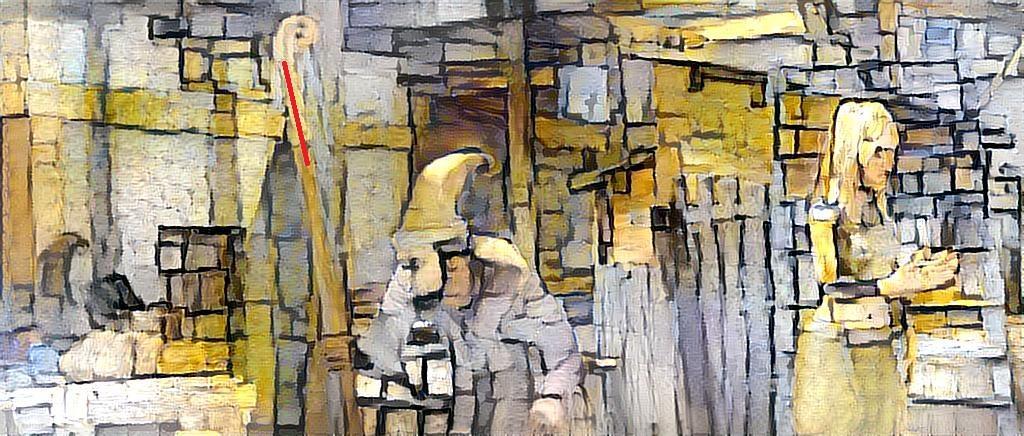}
    \includegraphics[trim={200px 240px 650px 10px},width=\ratio\textwidth,clip=true]{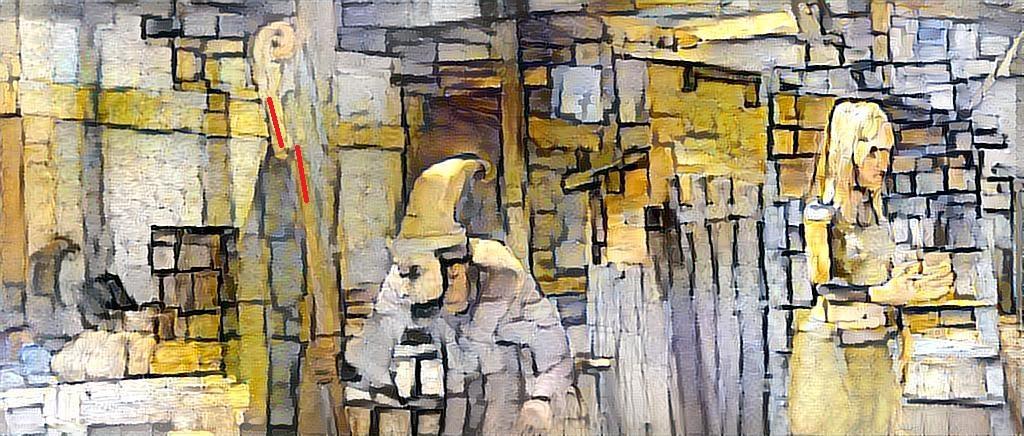}
    \includegraphics[width=23mm, height=24mm]{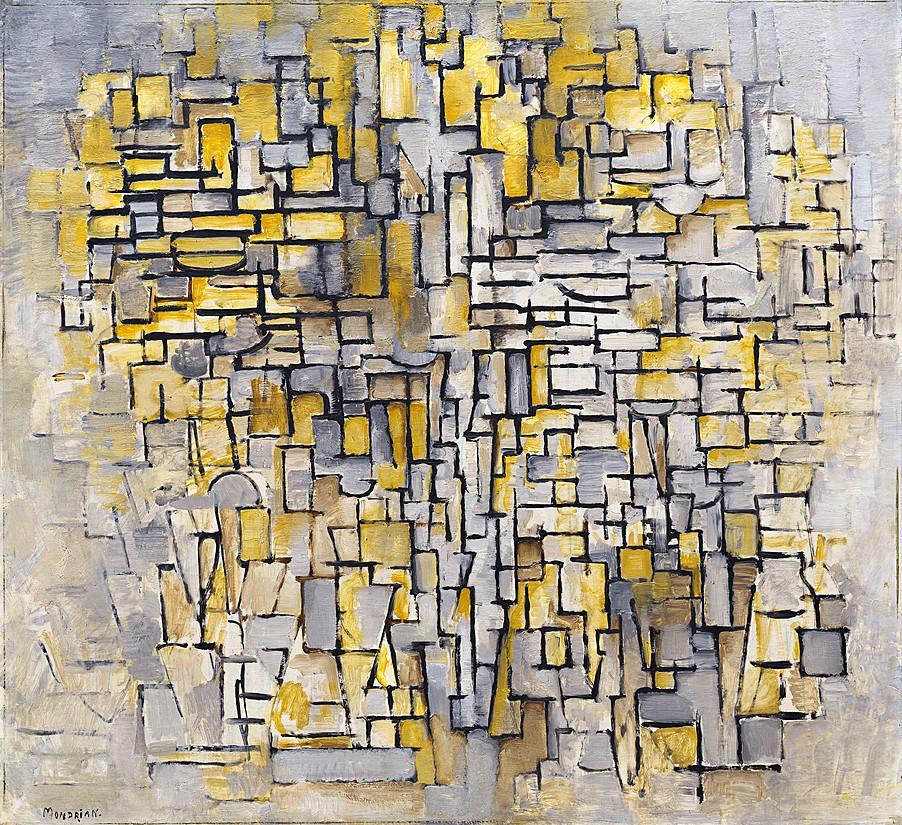}\\
        \raisebox{11mm}{
    \begin{minipage}[c][21mm][c]{3mm}
      \centering
      \rotatebox{90}{\centering \textbf{Ours}}
    \end{minipage}
  }
    \includegraphics[trim={200px 240px 650px 10px},width=\ratio\textwidth,clip=true]{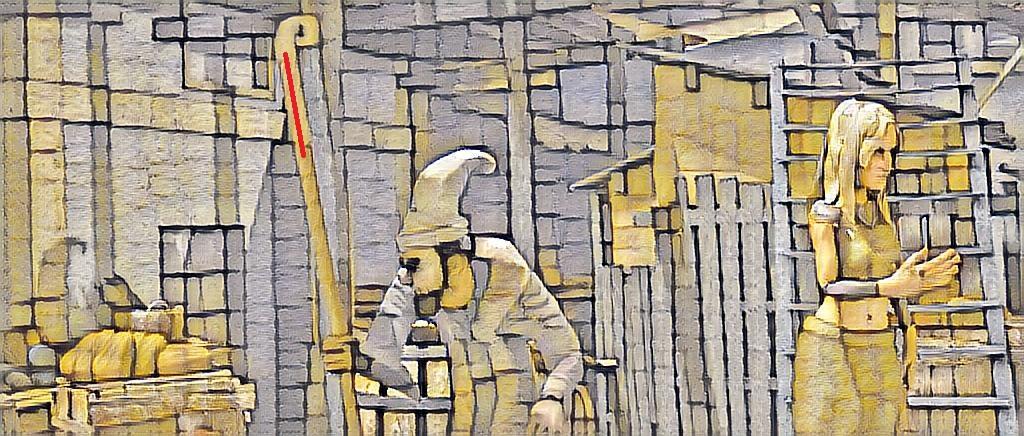} 
    \includegraphics[trim={200px 240px 650px 10px},width=\ratio\textwidth,clip=true]{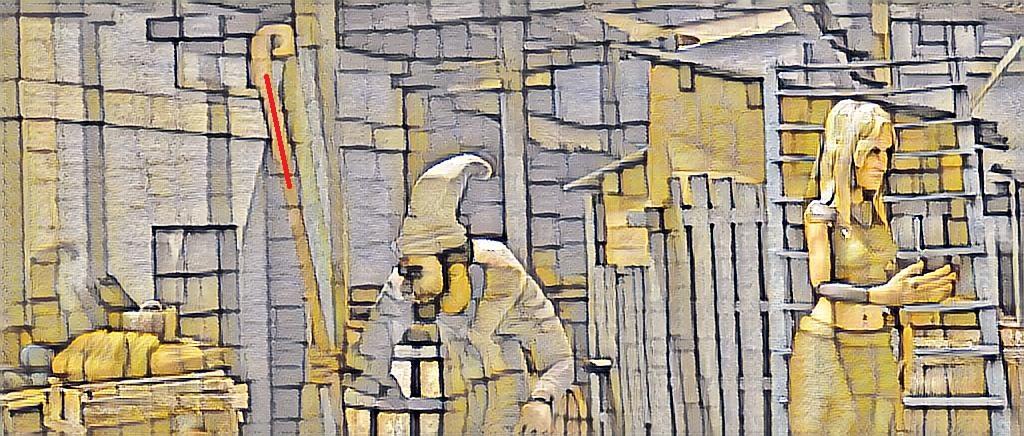} 
    \includegraphics[trim={200px 240px 650px 10px},width=\ratio\textwidth,clip=true]{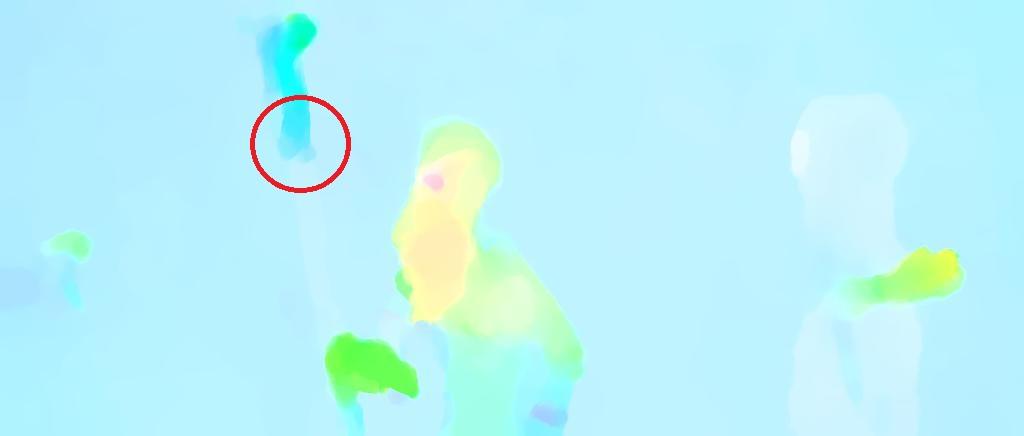}\\
     \raisebox{1mm}{
    \begin{minipage}[c][0.1mm][c]{3mm}
    \end{minipage}
  }
  \begin{minipage}{0.15\textwidth}
  \centering \hspace{0.5\textwidth}$t_1$
  \end{minipage}
    \begin{minipage}{0.15\textwidth}
	\centering $t_2$
  \end{minipage}
  \begin{minipage}{0.15\textwidth}
   Estimated Flow
  \end{minipage}
  }}\caption{Optim baseline is susceptible to errors in optical flow estimation. The estimated optical flow does not capture the motion of the bottom portion of the stick (Bottom-left: Dark blue represents more motion). Hence, at time $t_2$ the stylized image produced by Optim baseline shows the stick to be ``broken'' at the top. Our method does not require explicit optical flow at test time and does suffer from this failure mode (red lines added for emphasis).}
  \label{fig:optical-wrong}
\end{figure}

\def\mywidth{0.17\textwidth}
\def\ratio{0.15}
\begin{figure*}[ht!]
\raisebox{13mm}{
    \begin{minipage}{27mm}
      \centering
      \textbf{Style} \\
      \textit{Portrait de Metzinger},
      Robert Delaunay
    \end{minipage}
  }
    \includegraphics[height=0.13\textwidth]{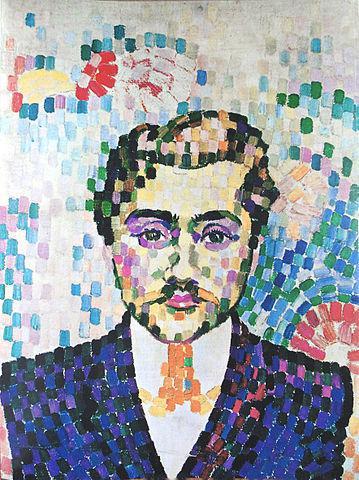}
  \hspace{7mm}
  \raisebox{15mm}{
    \begin{minipage}{25mm}
      \centering
      \textbf{Style} \\
      \textit{Rain Princess},\\
      Leonid Afremov
    \end{minipage}
  }
  \includegraphics[height=0.13\textwidth]{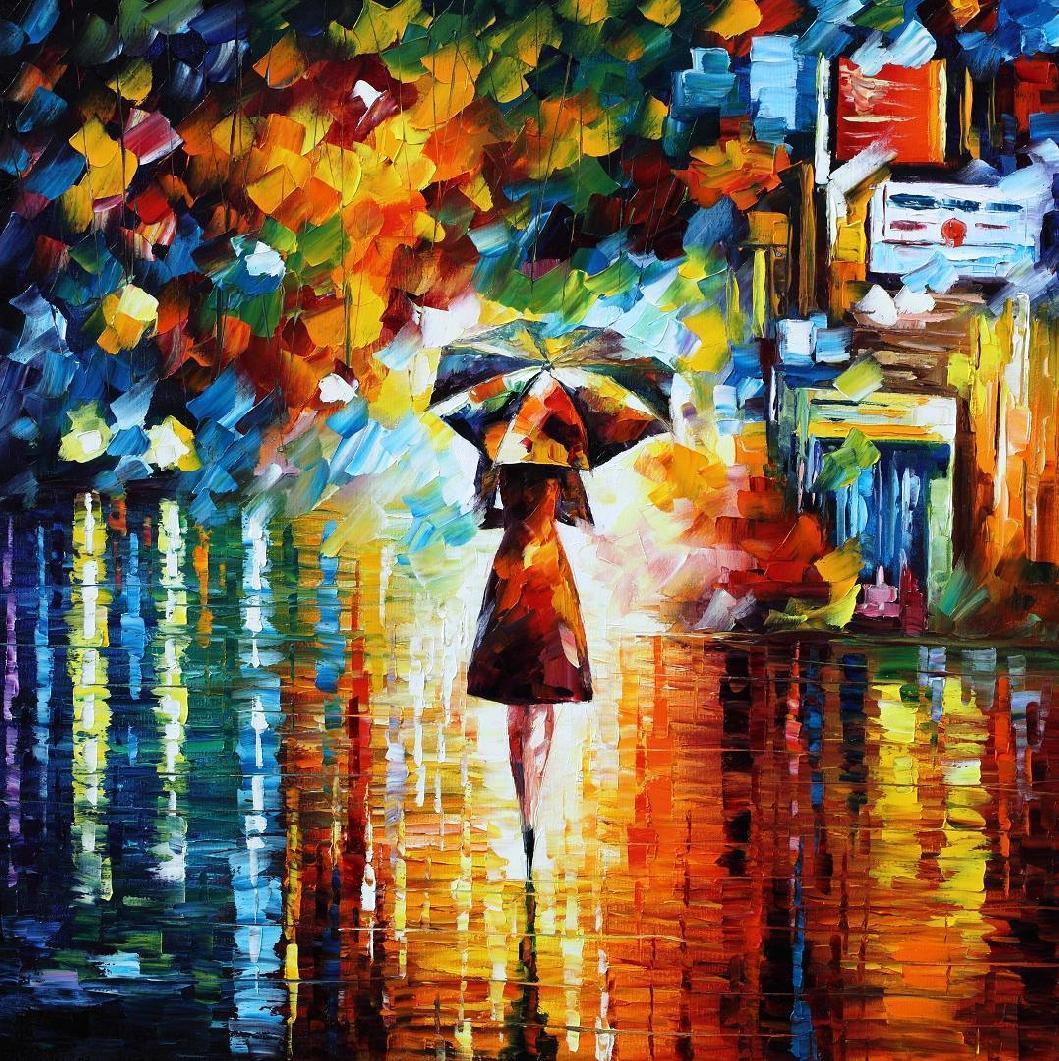}
   \hspace{7mm}
  \raisebox{11mm}{
    \begin{minipage}{25mm}
      \centering
      \textbf{Style} \\
      \textit{The Great Wave off Kanagawa},\\
      	Katsushika Hokusai
    \end{minipage}
  }
  \includegraphics[height=2.5cm,width=0.16\textwidth]{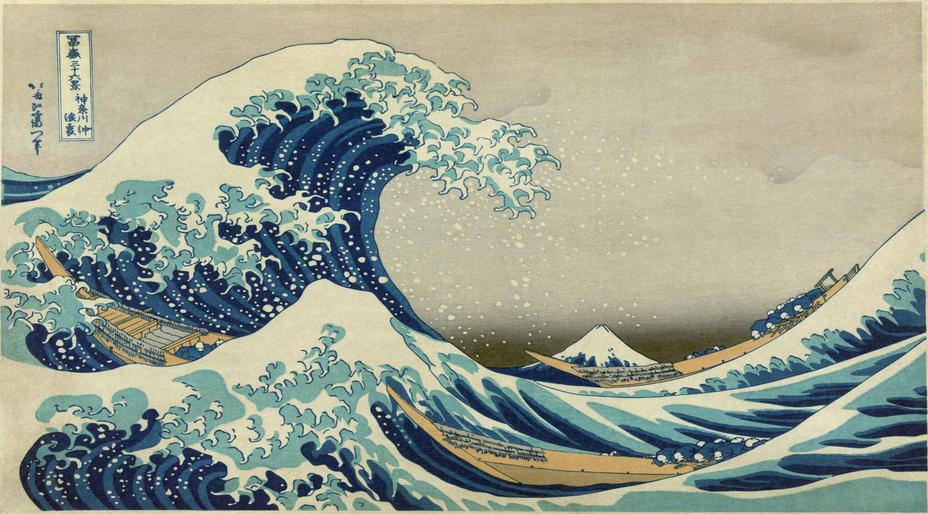}\\
  \hspace{7mm}
    \raisebox{9mm}{
    \begin{minipage}[c][10mm][c]{10mm}
      \centering
      \rotatebox{90}{\centering \textbf{Content}}
    \end{minipage}
  }
     \includegraphics[trim={670px 0px 0px 340px},width=\ratio\textwidth,clip=true]{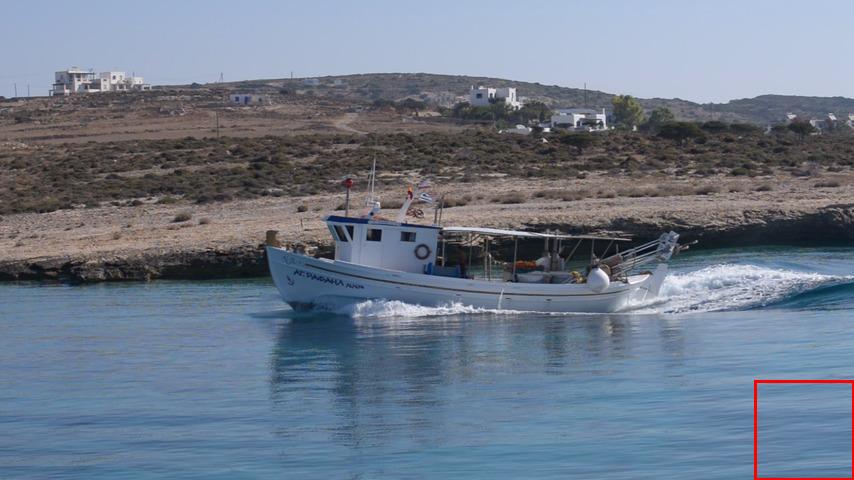}
     \includegraphics[trim={670px 0px 0px 340px},width=\ratio\textwidth,clip=true]{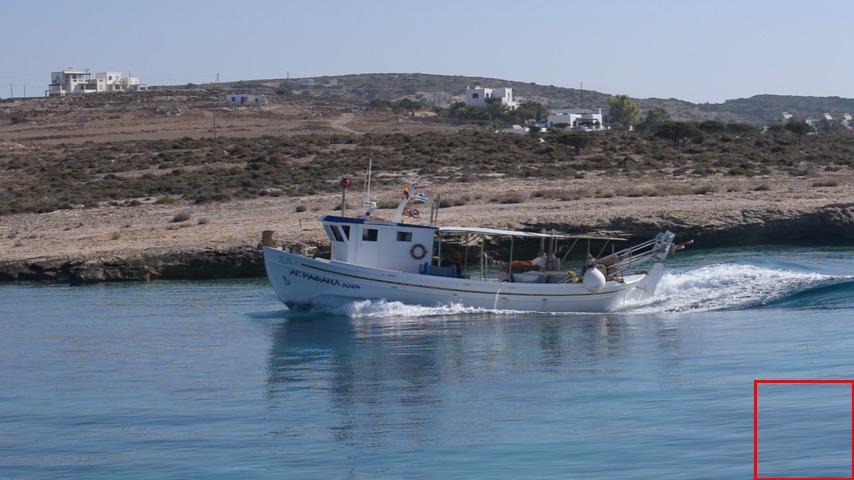} 
     \includegraphics[trim={20px 320px 650px 20px},width=\ratio\textwidth,clip=true]{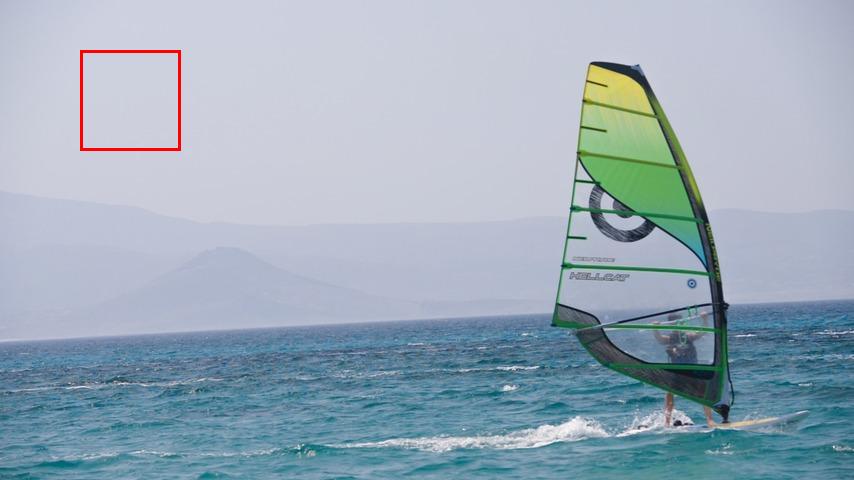}
     \includegraphics[trim={20px 320px 650px 20px},width=\ratio\textwidth,clip=true]{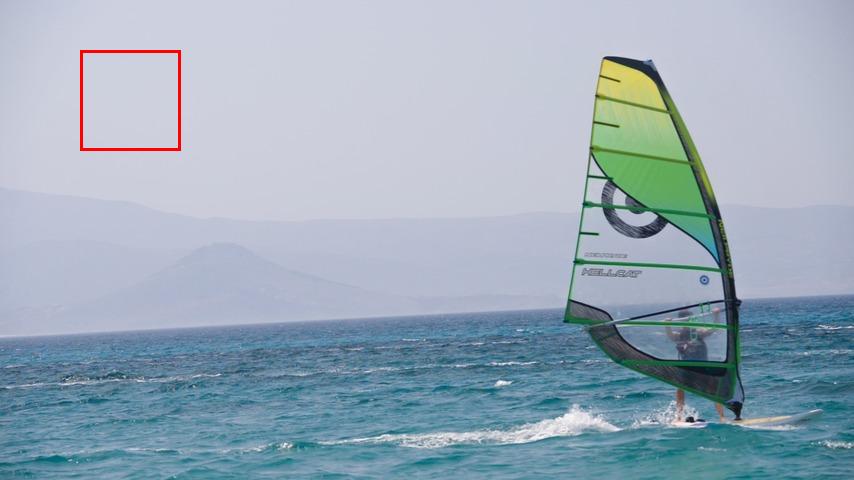} 
     \includegraphics[trim={600px 200px 70px 140px},width=\ratio\textwidth,clip=true]{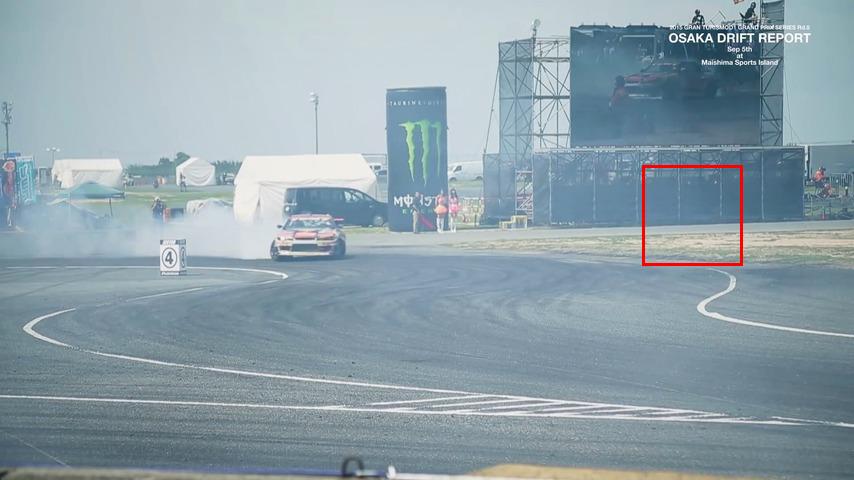}
     \includegraphics[trim={600px 200px 70px 140px},width=\ratio\textwidth,clip=true]{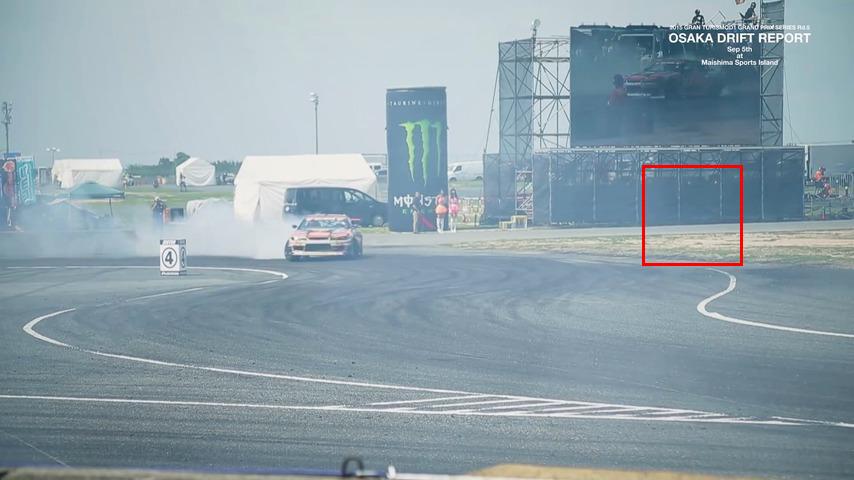}\\
  \hspace{7mm}
  \raisebox{9mm}{
    \begin{minipage}{10mm}
      \centering
      \rotatebox{90}{\parbox{1.6cm}{\centering \textbf{Optim} \\ \textbf{Baseline} \\ \cite{ruder2016artistic}}}
    \end{minipage}
  }
     \includegraphics[trim={670px 0px 0px 340px},width=\ratio\textwidth,clip=true]{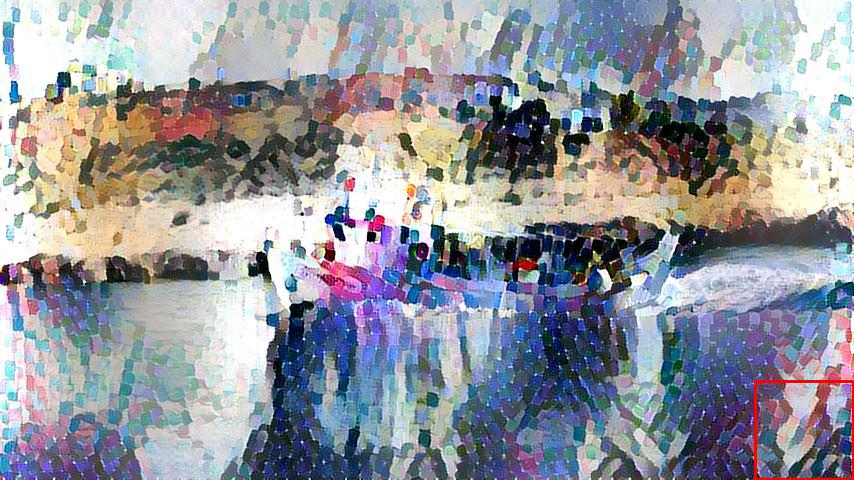}
     \includegraphics[trim={670px 0px 0px 340px},width=\ratio\textwidth,clip=true]{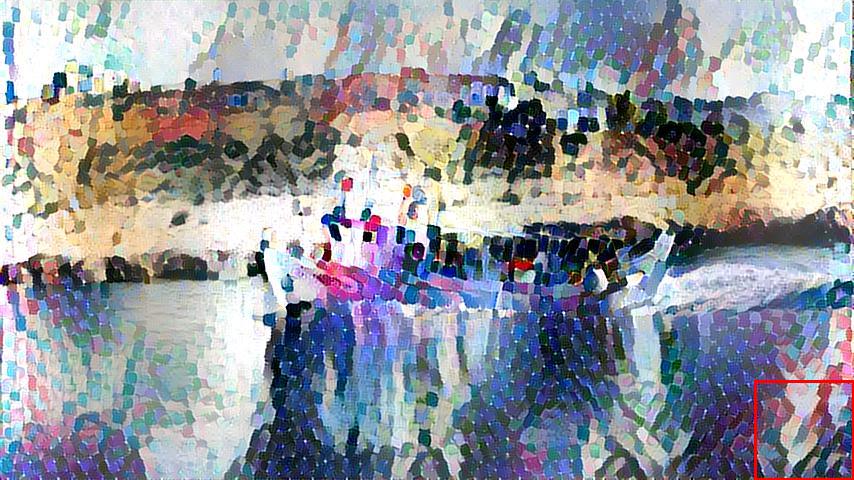}
     \includegraphics[trim={20px 320px 650px 20px},width=\ratio\textwidth,clip=true]{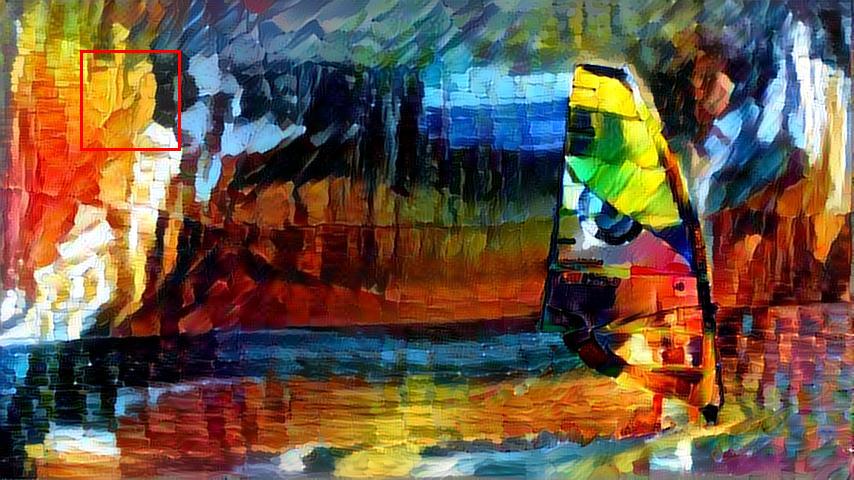}
     \includegraphics[trim={20px 320px 650px 20px},width=\ratio\textwidth,clip=true]{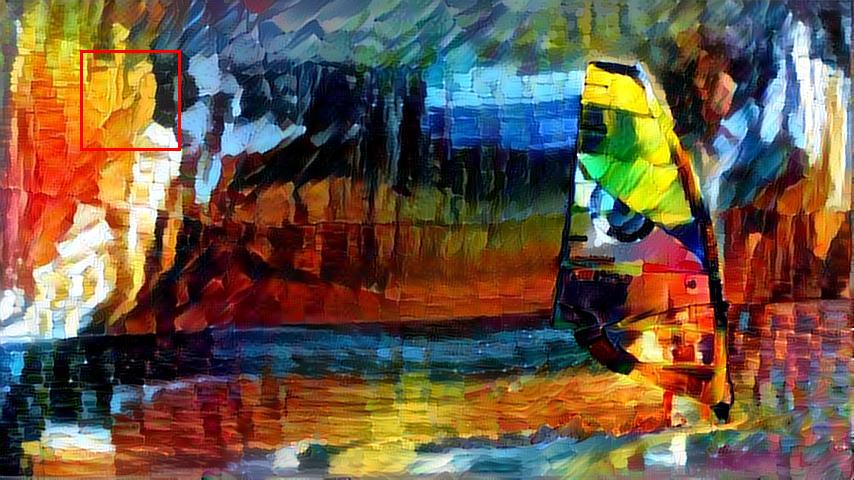} 
     \includegraphics[trim={600px 200px 70px 140px},width=\ratio\textwidth,clip=true]{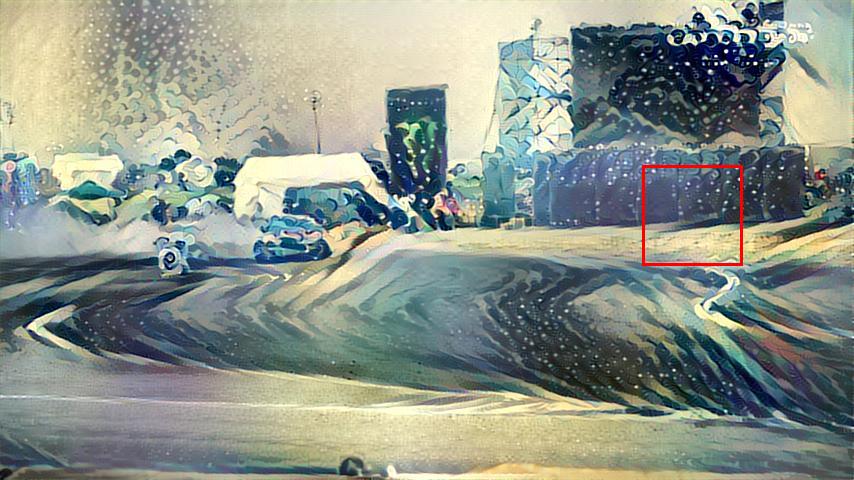}
     \includegraphics[trim={600px 200px 70px 140px},width=\ratio\textwidth,clip=true]{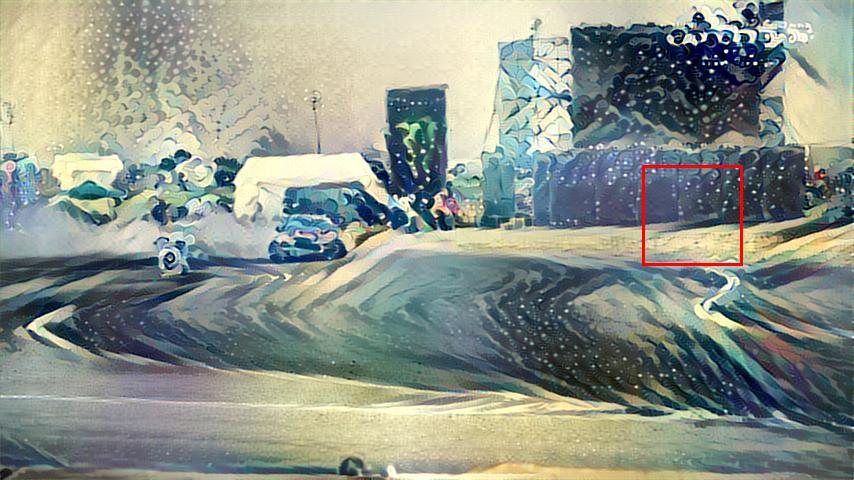} \\
  \hspace{7mm}
  \raisebox{9mm}{
    \begin{minipage}{10mm}
      \centering
      \rotatebox{90}{\parbox{1.6cm}{\centering \textbf{Real-Time} \\ \textbf{Baseline} \\ \cite{john}}}
    \end{minipage}
  }
     \includegraphics[trim={670px 0px 0px 340px},width=\ratio\textwidth,clip=true]{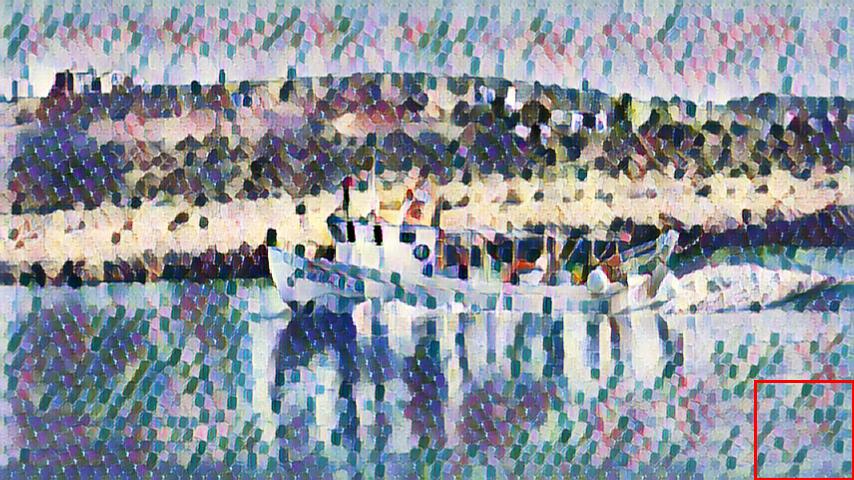}
     \includegraphics[trim={670px 0px 0px 340px},width=\ratio\textwidth,clip=true]{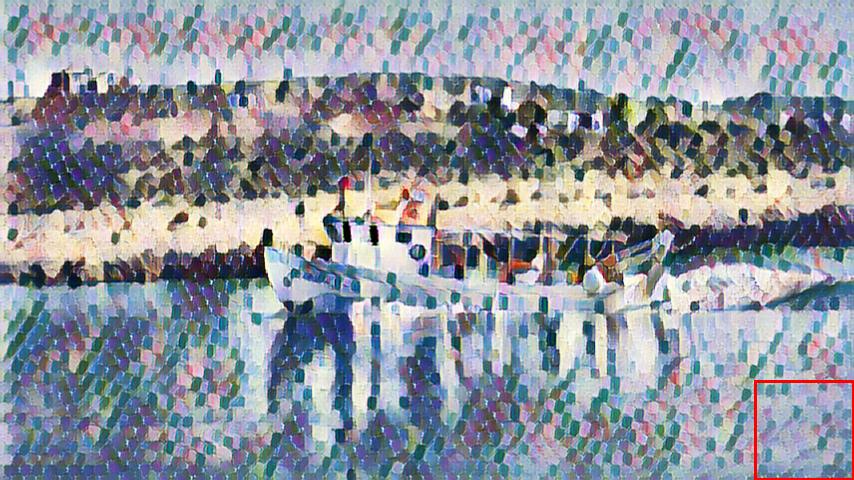}
     \includegraphics[trim={20px 320px 650px 20px},width=\ratio\textwidth,clip=true]{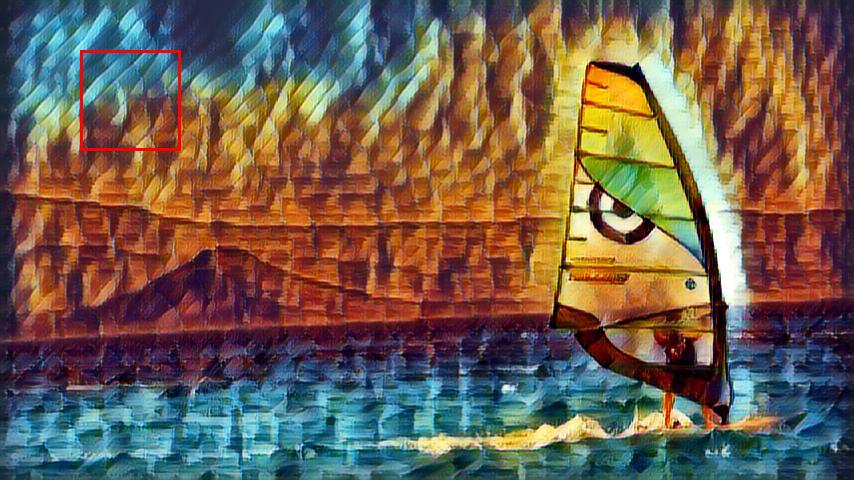}
     \includegraphics[trim={20px 320px 650px 20px},width=\ratio\textwidth,clip=true]{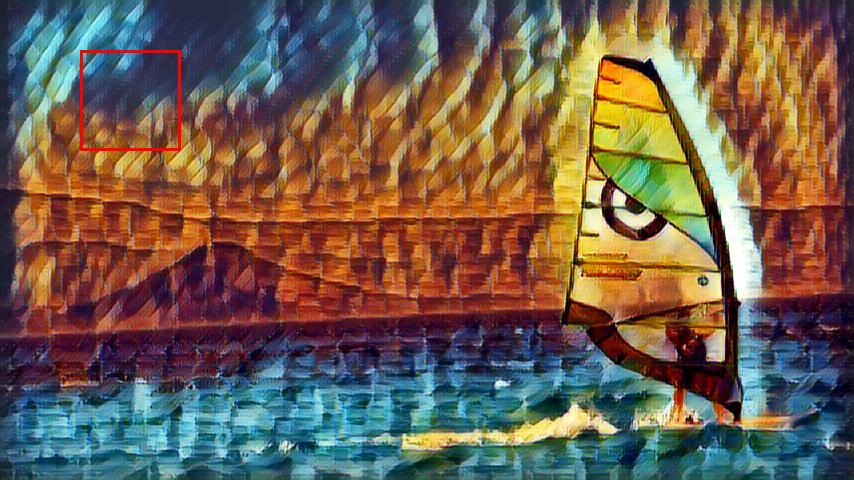} 
     \includegraphics[trim={600px 200px 70px 140px},width=\ratio\textwidth,clip=true]{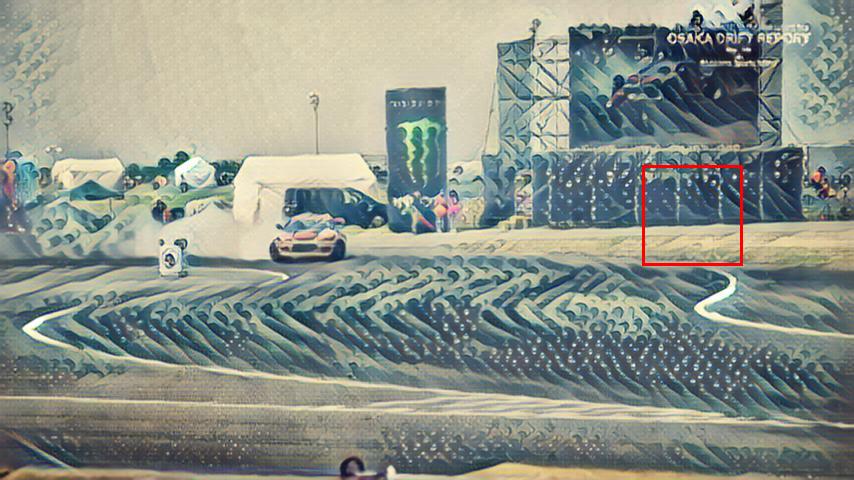}
     \includegraphics[trim={600px 200px 70px 140px},width=\ratio\textwidth,clip=true]{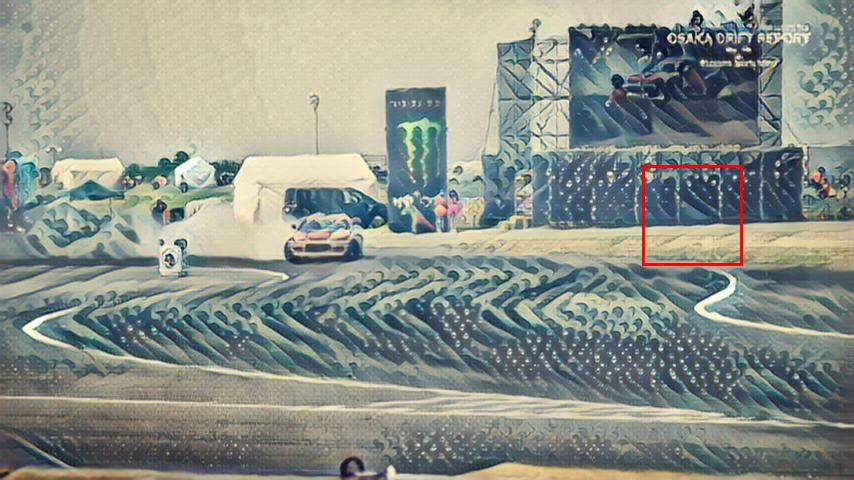} \\
  \hspace{7mm}
  \raisebox{9mm}{
    \begin{minipage}{10mm}
      \centering
      \rotatebox{90}{\parbox{1.6cm}{\centering \textbf{Ours}}}
    \end{minipage}
  }
     \includegraphics[trim={670px 0px 0px 340px},width=\ratio\textwidth,clip=true]{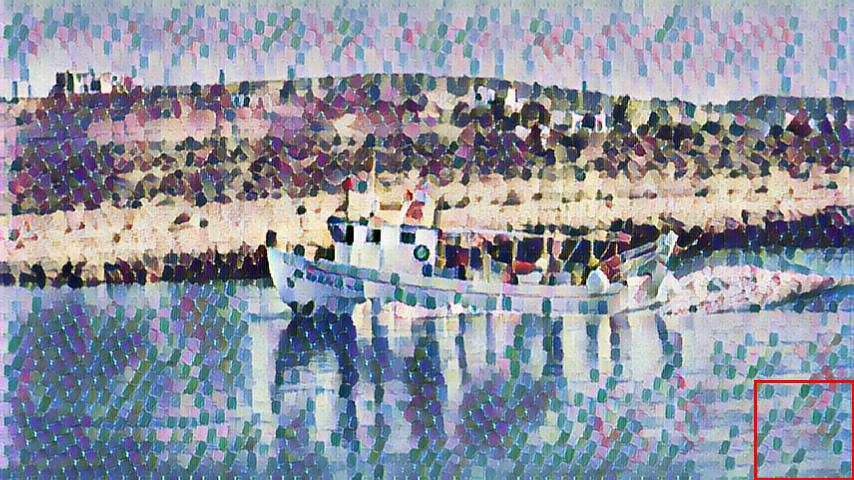}
     \includegraphics[trim={670px 0px 0px 340px},width=\ratio\textwidth,clip=true]{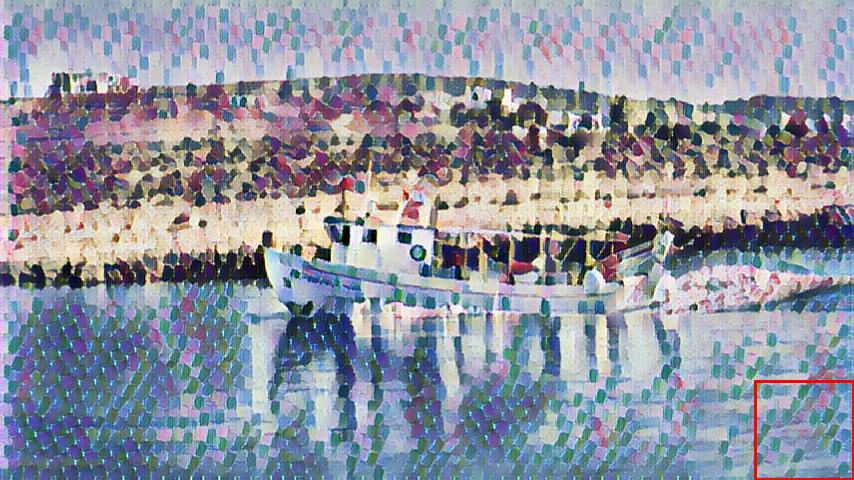}
     \includegraphics[trim={20px 320px 650px 20px},width=\ratio\textwidth,clip=true]{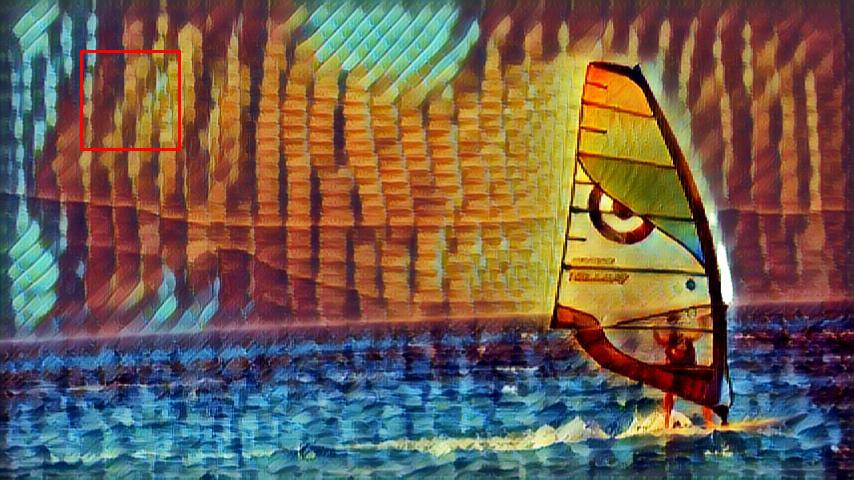}
     \includegraphics[trim={20px 320px 650px 20px},width=\ratio\textwidth,clip=true]{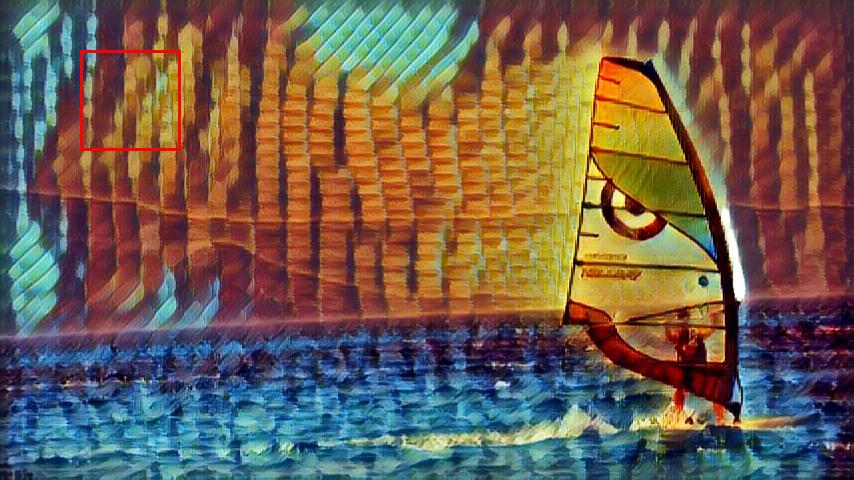} 
     \includegraphics[trim={600px 200px 70px 140px},width=\ratio\textwidth,clip=true]{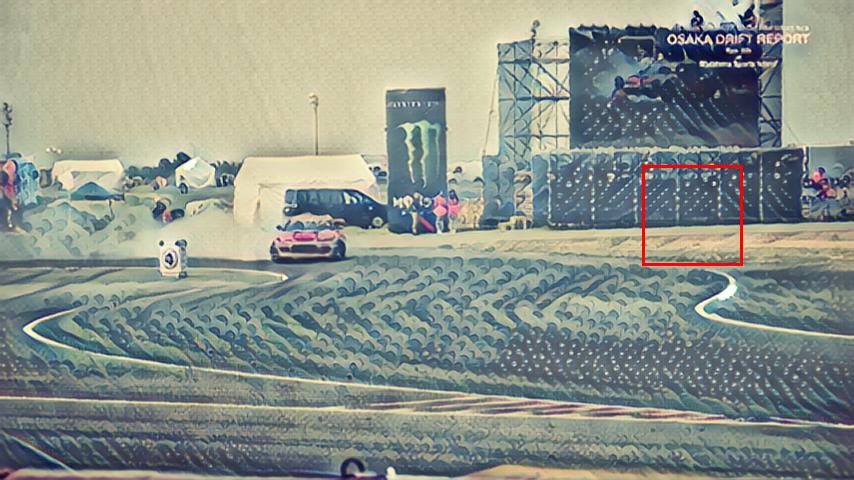}
     \includegraphics[trim={600px 200px 70px 140px},width=\ratio\textwidth,clip=true]{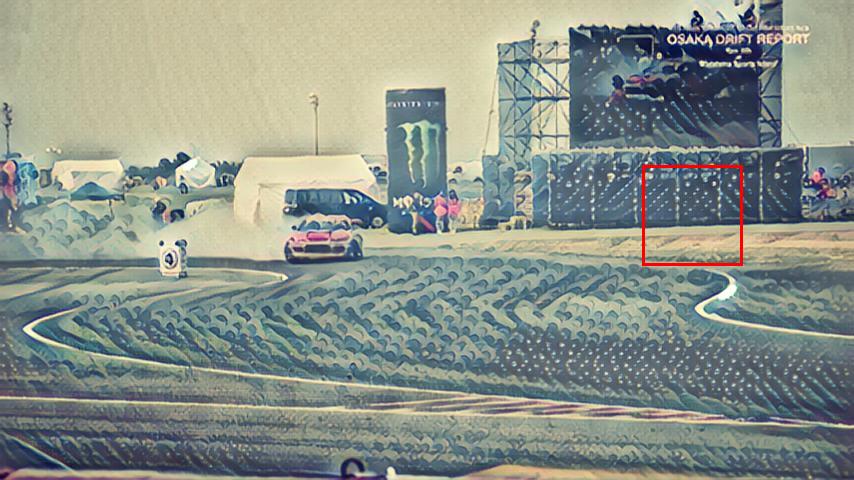} \\[2mm]
\begin{minipage}{0.16\textwidth}
    \centering \textbf{Method} \\ \cite{ruder2016artistic} \\ \cite{john} \\ Ours
  \end{minipage}
  \hspace{-0.04\textwidth}
  \begin{minipage}{0.20\textwidth}
    \centering \textbf{PSNR/SSIM} \\ 23.67 / 0.2078 \\ 23.47 / 0.1856 \\ \textbf{24.14} / \textbf{0.5234}
  \end{minipage}
  \hspace{0.09\textwidth}
    \begin{minipage}{0.20\textwidth}
    \centering \textbf{PSNR/SSIM} \\ \textbf{23.68} / 0.9136 \\ 23.13 / 0.1622 \\ 23.44 / \textbf{0.9240} 
  \end{minipage}
  \hspace{0.10\textwidth}
  \begin{minipage}{0.20\textwidth}
    \centering \textbf{PSNR/SSIM} \\ \textbf{32.53} / \textbf{0.9728} \\ 26.00 / 0.7432 \\ 27.76 / 0.9194
  \end{minipage} \\
  \caption{Examples of consecutive pairs of stylized video frame output. Our method produces stylistically similar frame sequences like Optim baseline. We report the PSNR/SSIM values for each example crop shown. Our method is significantly better in producing temporally consistent frames than real-time baseline for highly unstable styles like \emph{Rain Princess} and \emph{Metzinger}.}
  \label{fig:short-term}
\end{figure*}

\textbf{Transparency.} Our method sometimes gives results with slight object transparency when one object occludes another; for an example see the Supplementary Material. However, this effect typically only occurs for the first frame or two of object occlusion, and thus not very perceptible when viewing videos.

\textbf{Dependency on Optical Flow.} The Optim baseline requires forward and backward optical flow at test time to generate occlusion masks; failures in optical flow estimation can result in poor stylization.

Figure \ref{fig:optical-wrong} shows a crop of two frames from the \emph{Market\_1} video from the Sintel test set. Figure \ref{fig:optical-wrong} (bottom-right) shows the optical field as estimated by state-of-the-art optical flow algorithm \cite{weinzaepfel2013deepflow}. The estimated optical flow incorrectly estimates that only the top portion of the stick is moving. This error propagates to the stylized image: unlike the original pair of images where the whole stick moves, only the top portion of the stick to moves forward in the output produced by the  Optim baseline. Our method does not require ground truth optical flow during test time and consequently does not suffer from this failure mode. \par 

\textbf{User Study.}
We performed a user study on Amazon Mechanical Turk to compare the subjective quality
of our method and the two baselines. In each trial a worker is shown a video from the DAVIS dataset~\cite{Perazzi2016}, a style image, and stylized output videos from two methods. In each trial the worker answers three questions: \emph{``Which video flickers more?''}, \emph{``Which video better matches the style?''}, and \emph{``Overall, which video do you prefer?''}. For each question, the worker can either choose a video or select \emph{``About the same''}. Results are shown in Table \ref{table:userstudy}.

\begin{table}
\centering
\scalebox{1}{
    \begin{tabular}{|c|cc|cc|}
    	\hline
    	& \multicolumn{2}{c|}{Ours vs RT~\cite{john}}
        & \multicolumn{2}{c|}{Ours vs Optim~\cite{ruder2016artistic}} \\
    	Question & Ours & RT & Ours & Optim  \\
        \hline
        More Flicker & \textbf{26} & 193 & 6 & \textbf{4} \\
        Style Match & 70 & \textbf{114} & \textbf{8} & 4 \\
        Overall Prefer & \textbf{133} & 80 & 7 & \textbf{8} \\
        \hline
    \end{tabular}
}
\vspace{1mm}
\caption[caption for LOF]{Summary of user study results. We use 50 videos per style to evaluate our method against the Real-Time (RT) baseline~\cite{john} and 3 videos\protect\footnotemark~to evaluate against the Optim baseline~\cite{ruder2016artistic}. Each pair of videos is evaluated by five workers on Amazon Mechanical Turk. We report the number of videos where the majority of workers preferred one method over another for a particular question. Values in bold are better.} \label{table:userstudy}
\end{table} 
\footnotetext{Again, running the Optim baseline for all 50 videos is infeasible.}

Taken as a whole, this user study shows that our method results in videos with significantly less qualitative flickering than the Real-Time baseline, with temporal stability almost on par with the slower Optim baseline. Our method is perceived to match the style image about as well as other methods, and users prefer the results from our method significantly more often than the Real-Time baseline. We refer the reader to  Supplementary material for further details.

\section{Conclusion}
We studied the stability of the recent style transfer methods based on neural networks. We characterized the instability of style transfer methods based on Gram matrix matching by examining the solution set of the style transfer objective, arguing that instability is exacerbated when the trace of the Gram matrix of the
style image is large.

We then proposed a recurrent convolutional network for real-time video style transfer which overcomes the instability of prior methods. As future work, we want to investigate methods which can encode long-term consistency while maintaining real-time performance, ensuring stylistic consistency of objects which get occluded and then reappear over a sequence of consecutive video frames.

\pagebreak

\paragraph{Acknowledgments.}We thank Jayanth Koushik for helpful comments and discussions. This work is partially supported by an ONR MURI grant.

{\small
\bibliographystyle{ieee}
\bibliography{egbib}
}

\end{document}